\documentclass[11pt]{article}
\usepackage[margin=1in]{geometry} 
\usepackage{amsmath,amsthm,amssymb,amsfonts,color,graphicx,mathtools,bbm}
\usepackage{authblk}
\usepackage[mathscr]{euscript}
\usepackage{cases}
\usepackage{comment}
\usepackage[authoryear,sort&compress]{natbib}
\usepackage{tikz}
\newcommand*\circled[1]{\tikz[baseline=(char.base)]{
            \node[shape=circle,draw,inner sep=2pt, scale=.75] (char) {#1};}}
 
\newtheorem{theorem}{Theorem}[section]
\newtheorem*{theorem*}{Theorem}
\newtheorem{lemma}[theorem]{Lemma}

\newtheorem{proposition}[theorem]{Proposition}
 
\newtheorem{corollary}[theorem]{Corollary} 
\newtheorem{assumption}{Assumption} 
\theoremstyle{remark}
\newtheorem{remark}{Remark}

\newcommand{\R}{\mathbb{R}}
\newcommand{\bb}{\mathbb}
\newcommand{\Cal}{\mathcal}

\newcommand{\oh}{\otimes_{\mathcal{H}}}

\newcommand{\ohm}{\otimes_{\mathcal{H}_m}}
\newcommand{\X}{\mathcal{X}}

\newcommand{\HSH}{{\Cal{L}^2(\mathcal{H})}}

\newcommand{\OPH}{{\Cal{L}^\infty(\mathcal{H})}}

\newcommand{\OPRn}{{\Cal{L}^\infty(\mathbb{R}^n)}}

\newcommand{\Pb}{\mathbb{P}}
\newcommand{\Hk}{\mathcal{H}}
\newcommand{\kbar}{\bar{k}(\cdot,X)}

\newcommand{\id}{\mathfrak{I}}

\newcommand{\Var}{\text{Var}}

\newcommand{\inner}[2]{\left\langle #1,#2\right\rangle}
\newcommand{\norm}[1]{\left\lVert#1\right\rVert}

\newcommand{\iid}{i.\,i.\,d.}

\newcommand{\E}{\mathbb{E}}

\newcommand{\bm}{\mathbf}

\newcommand{\sigh}{\widehat{\Sigma}}
\newcommand{\lambdah}{\widehat{\lambda}}
\newcommand{\lambdatild}{\Tilde{\lambda}}
\newcommand{\phih}{\widehat{\phi}}
\newcommand{\phitild}{\Tilde{\phi}}

\newcommand{\lp}{{L^2(\mathbb{P})}}

\newcommand{\intx}{\int_{\mathcal{X}}}

\newcommand{\HS}{\mathcal{L}^2}
\newcommand{\Tr}{\mathcal{L}^1}

\newcommand{\vp}{\varphi}

\title{Statistical Optimality and Computational Efficiency of Nystr\"{o}m Kernel PCA}
\author{Nicholas Sterge}
\author{Bharath Sriperumbudur}
\affil{Department of Statistics,  
Pennsylvania State University\\
University Park, PA 16802, USA.\\
\texttt{\{nzs5368,bks18\}@psu.edu}}
\date

\begin{document}

\maketitle

\begin{abstract}

Kernel methods provide an elegant framework for developing nonlinear learning algorithms from simple linear methods. Though these methods have superior empirical performance in several real data applications, their usefulness is inhibited by the significant computational burden incurred in large sample situations. Various approximation schemes have been proposed in the literature to alleviate these computational issues, and the approximate kernel machines are shown to retain the empirical performance. However, the theoretical properties of these approximate kernel machines are less well understood. In this work, we theoretically study the trade-off between computational complexity and statistical accuracy in Nystr\"om approximate kernel principal component analysis (KPCA), wherein we show that the Nystr\"om approximate KPCA matches the statistical performance of (non-approximate) KPCA while remaining computationally beneficial. Additionally, we show that Nystr\"om approximate KPCA outperforms the statistical behavior of another popular approximation scheme, the random feature approximation, when applied to KPCA.
\end{abstract}
\setlength{\parskip}{4pt}


\section{Introduction}
Principal component analysis (PCA) \citep{Jollife-86} is an unsupervised learning technique in which a random variable $X$ taking values in $\R^d$ is projected onto the direction $\mathbf{a}\in\R^d$ such that $\text{Var}[\mathbf{a}^\top X]$ is maximized. Further, for some $\ell<d$, PCA may be used to find an $\ell$-dimensional subspace retaining the maximum possible variance of $X$, making PCA a popular methodology for dimension reduction and feature extraction. This low-dimensional subspace is the $\ell$-eigenspace, i.e., the span of the eigenvectors associated with the top $\ell$ eigenvalues of the covariance matrix $\E XX^\top-\E X\E X^\top$. The respective eigenvectors are referred to as the principal components of the data, and a lower-dimensional representation of the input data may be computed by projecting onto the principal components. 

\par The principal components outputted by PCA are linearly related to the original coordinates; however, in many cases a non-linear component provides a better description of the data. Kernel PCA (KPCA) \citep{Scholkopf-98} is a non-linear extension of PCA which maps the original data into a reproducing kernel Hilbert space (RKHS) \citep{Aronszajn-50} where PCA is performed, resulting in principal components which are non-linearly related to the original data. Specifically, for an RKHS $\Hk$ with reproducing kernel $k:\X\times\X\rightarrow\R$, KPCA solves $\sup\{\text{Var}[f(X)]:\norm{f}_\Hk=1\}$.  Analogous to linear PCA, the principal components in KPCA are the eigenfunctions of the covariance operator $\Sigma=\E[\Phi(X)\oh\Phi(X)]-\E\Phi(X)\oh\E\Phi(X)$, where $\Phi(X)=k(\cdot,X)$ is the \textit{feature map}. Similarly, an $\ell$-dimensional representation of $X$ is obtained by projecting onto the $\ell$-eigenspace of $\Sigma$. Kernel PCA has been employed successfully in tasks such as computer vision \citep{Lampert-09}, image denoising \citep{Mika-99}, and other learning environments with complex spatial structures.

\par Empirically, given $X_1,\ldots,X_n\stackrel{\iid}{\sim}\Pb$, the eigenfunctions of $\Sigma$ are estimated by those of the empirical covariance operator \begin{equation}\sigh=\frac{1}{2n(n-1)}\sum_{i\neq j}^n\left(\Phi(X_i)-\Phi(X_j)\right)\oh\left(\Phi(X_i)-\Phi(X_j)\right),\label{Eq:sigh1}\end{equation} yielding empirical KPCA (EKPCA). Though this may require solving an infinite dimensional system, it can be shown (see Proposition~\ref{prop:ekpca soln}) that the eigenfunctions of $\sigh$ can be obtained by solving an $n$-dimensional eigenvalue problem, which has a computational requirement of $O(n^3)$ and a memory requirement of $O(n^2)$. This means, EKPCA scales quite poorly with large sample sizes, a behavior shared by many kernel methods. This has lead to a lot of research activity in constructing approximation methods which relieve the computational burden. Nystr\"{o}m method \citep{Reinhardt-85, Williams-98} is a popular approximation scheme, which
uses a subsample of the original data to construct a low-rank approximation to the Gram matrix $\mathbf{K}=[k(X_i,X_j)]_{i,j}$, which in turn is closely related to $\widehat{\Sigma}$. More precisely, the Gram matrix $\mathbf{K}$ is approximated by 
$$\tilde{\mathbf{K}}=\mathbf{K}_{nm}\mathbf{K}_{mm}^{-1}\mathbf{K}_{nm}^\top,$$
where $\mathbf{K}_{nm}$ is the matrix formed by randomly selecting $m$ columns of $\mathbf{K}$ and $\mathbf{K}_{mm}$ is the intersection of those $m$ columns and $m$ rows of $\mathbf{K}$. The approximate Gram matrix can be used in subsequent learning tasks, resulting in computational saving for $m<n$. In the case of KPCA, the computational complexity is reduced from $O(n^3)$ to $O(nm^2)$. Of course, the question of interest is whether this computational saving comes at the cost of statistical accuracy. In this work we establish the consistency of Nystr\"om KPCA (NY-EKPCA), and study the relationship between statistical behavior and computational complexity.

\subsection{Contributions}\label{sec:contributions}
The contributions of the paper are as follows:
\vspace{2.5mm}

\noindent \emph{(i)} In Section \ref{sec:NY-EKPCA}, we propose Nystr\"om empirical KPCA (NY-EKPCA), and demonstrate its computational complexity. In Section \ref{sec:NY-EKPCA main results}, we compare the performance of empirical KPCA (EKPCA) with that of NY-EKPCA in terms of the reconstruction error of the $\ell$-eigenspaces in $\Hk$. We show that NY-EKPCA matches the statistical performance of EKPCA with less computational complexity, provided the number of subsamples, $m$, is large enough, and the number of eigenfunctions used in the reconstruction, $\ell$, is not too large. We note that similar analysis has been performed working with the uncentered covariance operator \citep{Sterge-20}, that is, $\sigh$ is defined as a V-statistic estimator where the mean element is assumed to be zero ($\E_{X\sim\bb{P}} k(\cdot,X)=0$), i.e., $\sigh=\frac{1}{n}\sum^n_{i=1}k(\cdot,X_i)\oh k(\cdot,X_i)$. This assumption of $\E_{X\sim\bb{P}} k(\cdot,X)=0$ is highly restrictive, as it is not satisfied by virtually all common kernels, e.g., Gaussian, Mat\'ern, inverse multiquadric, that induce an infinite dimensional RKHS. However, if this assumption is relaxed, the resulting V-statistic estimator, i.e., $$\frac{1}{n}\sum^n_{i=1}k(\cdot,X_i)\oh k(\cdot,X_i)-\left(\frac{1}{n}\sum^n_{i=1}k(\cdot,X_i)\right)\oh \left(\frac{1}{n}\sum^n_{i=1}k(\cdot,X_i)\right)$$ is no longer unbiased. Since unbiasedness is crucial for a tighter analysis, we consider a U-statistic estimator of $\Sigma$ as shown in \eqref{Eq:sigh1} and develop the analysis based on Bernstein-type inequality for operator valued U-statistics, which we proposed in our earlier work \citep{sriperumbudur-20}. Thus, the current work provides a non-trivial extension to our previous results by relaxing a significant assumption of $\E_{X\sim\bb{P}} k(\cdot,X)=0$.

\noindent \emph{(ii)} In Section \ref{sec:main results lp}, to foster a comparison with random features approximation \citep{Rahimi-08a, sriperumbudur-20}---another popular kernel approximation---, we study the performance of NY-EKPCA in terms of the reconstruction error of the $\ell$-eigenspace in $\lp$. Comparing NY-EKPCA with EKPCA and random feature approximate EKPCA (RF-EKPCA), we show that NY-EKPCA again recovers the statistical performance of EKPCA with less computational complexity. However, unlike in $\Hk$ where the number of eigenfunctions, $\ell$ used in the reconstruction cannot be too large, the result in $\lp$ holds regardless of the number of eigenfunctions, $\ell$, used in the reconstruction. Additionally, we show the Nystr\"{o}m approximation to be superior to that of the random features approximation by showing that NY-EKPCA outperforms RF-EKPCA in terms of the reconstruction error while enjoying better computational complexity---similar observation was already made in the context of kernel ridge regression \citep{Rudi-15,Rudi-17}.

\subsection{Related Work}
The statistical behavior of EKPCA has been well studied. The statistical consistency of EKPCA is established by \cite{Shawe-Taylor-05}, where the reconstruction error of the empirical $\ell$-eigenspace is shown to converge at the rate $\sqrt{\ell/n}$. \cite{Blanchard-07} and \cite{Rudi-13} obtain improved rates by considering the decay rate of the eigenvalues of the covariance operator. With the exception of \cite{Blanchard-07}, these works consider uncentered KPCA, though empirical recentering of the data is often performed in practice.

\par Outside of the Nystr\"om method, several other approximation strategies have been proposed in the kernel methods literature. These include incomplete Cholesky \citep{Fine-01, Bach-05a}, random features \citep{Rahimi-08a}, sketching \citep{Yang-17},  and sparse greedy approximation \citep{Smola-00}. 
These methods, including Nystr\"om, offer significant reductions in computational complexity, and, empirically, have been shown to provide performance competitive to their more expensive counterparts without approximation \citep{rahimi08weighted, Kumar-09, Yang-12}. Theoretical analysis of the Nystr\"om method has primarily concerned the distance between the Gram matrix, $\textbf{K}$, and its low-rank approximation (\citealp{Drineas-05}; \citealp{Gittens-13}; \citealp{Jin-13}). Recent research has focused on the impact of Nystr\"om approximation in specific learning tasks, allowing one to observe the trade-off between statistical accuracy and computational complexity. The supervised setting has been studied heavily \citep{Alaoui-15, Rudi-15, Bach-13}, where it has been shown that Nystr\"om approximation can achieve best possible statistical performance with better computational complexity. Significantly less is known in the unsupervised setting; however, approximate KPCA has been studied \citep{Lopez-Paz-14, streaming_kpca}. \cite{sriperumbudur-20} show that approximate KPCA using random features achieves better computational complexity than EKPCA with no loss in statistical performance. 
Recently, \cite{Sterge-20} show that Nystr\"om approximate KPCA can yield computational benefit without statistical loss; however, as mentioned in Section \ref{sec:contributions}, their result hinges on the highly-restrictive assumption that $\E_{X\sim\Pb}k(\cdot,X)=0$.

\section{Definitions and Notation}\label{Sec:notation}
\indent For $\bm{a}:=(a_1,\ldots,a_d)\in\bb{R}^d$ and $\bm{b}:=(b_1,\ldots,b_d)\in\bb{R}^d$ define $\Vert
\bm{a}\Vert_2:=\sqrt{\sum^d_{i=1}a^2_i}$ and $\langle \bm{a},\bm{b}\rangle_2:=\sum^d_{i=1}a_ib_i$. 
$\bm{a}\otimes_2 \bm{b}:=\bm{a}\bm{b}^\top$ denotes the tensor product of $\bm{a}$ and $\bm{b}$. $\mathbf{I}_n$ denotes an $n\times n$ identity matrix and $\mathbf{1}_n=(1,\stackrel{n}{\hdots},1)^\top$. We define $\mathbf{C}_n=\mathbf{I}_n-\frac{1}{n}\mathbf{1}_n\mathbf{1}_n^\top$ and $\mathbf{H}_n=n\mathbf{C}_n$.  For a matrix $A\in\R^{n\times m}$, $A^+\in\R^{m\times n}$ denotes the \emph{Moore-Penrose generalized inverse} of $A$.  
$a\wedge b:=\min(a,b)$ and $a\vee b:=\max(a,b)$. Define $[n]:=\{1,\ldots,n\}$ for $n\in\bb{N}$. For constants $a$ and $b$, $a\lesssim b$ (\emph{resp.} $a\gtrsim b$) denotes that there exists a positive constant $c$ (\emph{resp.} $c'$) such that $a\le cb$ (\emph{resp.} $a\ge c'b$). For a random variable $A$ with law $P$ and a constant $b$, $A\lesssim_p b$ denotes that for any $\delta>0$, there exists a positive constant $c_\delta<\infty$ such that $P(A\le c_\delta b)\ge \delta$.

For $x,y\in
H$, a Hilbert space, $x\otimes_{H} y$ is an element of the tensor product space
$H\otimes H$ which can also be seen as an operator from $H$ to $H$ as
$(x\otimes_{H} y)z=x\langle y,z\rangle_{H}$ for any $z\in H$. $\alpha\in\bb{R}$ is called an \emph{eigenvalue} of a bounded self-adjoint operator $S$ 
if there exists an $x\ne 0$ such that $Sx=\alpha x$ and such an $x$ is called
the \emph{eigenvector}/\emph{eigenfunction} of $S$ and $\alpha$. An eigenvalue is said to be \emph{simple} if it has multiplicity one. For an operator $S:H\rightarrow H$, $\Vert S\Vert_{\Tr(H)}$,
$\Vert S\Vert_{\HS(H)}$ and $\Vert S\Vert_{\Cal{L}^\infty(H)}$ denote the trace, Hilbert-Schmidt and operator norms of $S$, respectively.

\section{Kernel PCA and its Variations}
\begin{assumption}\label{first kern assum} $(\Cal{X, B})$ is a completely separable space endowed with $\sigma$-algebra $\Cal{B}$.  $\Hk$ is a separable RKHS of $\R$-valued functions on $\X$ with bounded continuous positive definite kernel satisfying $\sup_{x\in\X}k(x,x)=:\kappa<\infty$.
\end{assumption}
The assumption that $\X$ is completely separable, also called second countable, ensures that $\Cal{B}$ is countably generated; therefore, $L^r(\X,\mu)$ is separable for any $\sigma$-finite measure $\mu$ on $\Cal{B}$ and $r\in[1,\infty)$ \citep[Proposition 3.4.5]{cohn-13}. The separability of $\Hk$ as well as $k$ being bounded and continuous guarantees that $k(\cdot,x):\X\rightarrow\Hk$ is Bochner-measurable for all $x\in\X$ \citep[Theorem 8]{dinculeanu-00}.

\subsection{Kernel PCA in the Population and Sample}
Kernel PCA \citep{Scholkopf-98} is an unsupervised learning method in which classical PCA is performed on data which has been mapped to a reproducing kernel Hilbert space.  That is, kernel PCA (KPCA) finds $f\in\Hk$ with unit norm such that $\Var[f(X)]=\E[f(X)-\E[f(X)]]^2$ is maximized.  Using the reproducing property, we have $\Var[f(X)]=\E[\inner{f}{k(\cdot,X)}_\Hk-\inner{f}{m_\Pb}_\Hk]^2$ where $m_\Pb\in\Hk$ is the unique \textit{mean element} of $\Pb$ in $\Hk$, defined for all $f\in\Hk$ by
\begin{equation}\label{Eq:mean element}
    \inner{f}{m_\Pb}_\Hk=\E[f(X)]=\E[\inner{f}{k(\cdot,X)}_\Hk]=\inner{f}{\int_\X k(\cdot,x)d\Pb(x)}_\Hk,
\end{equation}
where the last equality of (\ref{Eq:mean element}) holds via Riesz representation theorem \citep{Reed-80} and the boundedness of $k$ from Assumption \ref{first kern assum} which ensures $k(\cdot,X)$ is Bochner integrable \citep{Diestel-77} with respect to $\Pb$.  Therefore, we may write 
$\Var[f(X)]=\inner{f}{\Sigma f}_\Hk$
where
\begin{equation}\label{Def:Sigma}
    \Sigma=\int_\X k(\cdot,x)\oh k(\cdot,x)d\Pb(x)-m_\Pb\oh m_\Pb,
\end{equation}
is the covariance operator on $\Hk$ associated with $\Pb$. Thus, the KPCA problem may be expressed as
\begin{equation}\label{kpca obj}
    \sup\{\inner{f}{\Sigma f}_\Hk : f\in\Hk,\,\norm{f}_\Hk=1\},
\end{equation}
bearing a strong resemblance to classical PCA.  In fact, KPCA can be seen as a generalization of classical linear PCA, as taking $\Hk=\R^d$ with $k(x,y)=x^\top y$ yields classical PCA with covariance matrix $\Sigma=\E[XX^\top]-\E[X]\E[X]^\top$.  The boundedness of $k$ in Assumption~\ref{first kern assum} ensures that $\Sigma$ is trace class and thus compact.  Since $\Sigma$ is positive and self-adjoint, the spectral theorem \citep{Reed-80} gives
\begin{equation*}
    \Sigma=\sum_{i\in I}\lambda_i\phi_i\oh\phi_i,\nonumber
\end{equation*}
where $(\lambda_i)_{i\in I}\subset\R^+$  and $(\phi_i)_{i\in I}$ are the eigenvalues and eigenfunctions, respectively, of $\Sigma$.  $(\phi_i)_{i\in I}$ form an orthonormal system spanning $\overline{\Cal{R}(\Sigma)}$, where the index set $I$ is either finite or countable, in which case $\lambda_i\rightarrow 0$ as $i\rightarrow\infty$.  The solution to (\ref{kpca obj}) is simply the eigenfunction of $\Sigma$ corresponding to its largest eigenvalue. We make the following simplifying assumption for ease of presentation.
\begin{assumption}\label{Sigma eigen assum}
The eigenvalues $(\lambda_i)_{i\in I}$ of $\Sigma$ are simple, positive, and w.l.o.g.~satisfy a decreasing rearrangement, i.e., $\lambda_1>\lambda_2>\ldots$\vspace{-2mm}
\end{assumption}
\noindent Assumption \ref{Sigma eigen assum} allows one to express the orthogonal projection operator onto the $\ell$-eigenspace of $\Sigma$, i.e. span$\{(\phi_i)_{i=1}^\ell\}$, as \vspace{-2mm}
\begin{equation}\label{kpca projector}
P^\ell(\Sigma)=\sum_{i=1}^\ell\phi_i\oh\phi_i.
\end{equation}

The above construction
corresponds to population version when the data distribution $\Pb$ is known.  In practice, the knowledge of $\Pb$ is available only through the sample $\{X_i\}_{i=1}^n\stackrel{i.i.d.}{\sim}\Pb$. Therefore, performing KPCA in practice requires one to replace $\Sigma$ in (\ref{Def:Sigma}) with an estimate.  In the vast majority of the literature (e.g., \citealt{Scholkopf-98, Shawe-Taylor-05}), the assumption $m_\Pb=0$ is made and the corresponding V-statistic estimator of $\Sigma$, i.e., $\frac{1}{n}\sum_{i=1}^n k(\cdot,X_i)\oh k(\cdot,X_i)$ is used.  However, the assumption $m_\Pb=0$ is quite restrictive as many popular kernels, such as the Gaussian, do not satisfy this condition.  
Therefore, to make the setting and results more general, we make no such assumption on $m_\Pb$; however, this relaxation causes the resultant V-statistic estimator to be biased. To mediate the technical difficulties arising from a biased estimator, we choose the U-statistic estimator,
\begin{equation*}
    \sigh=\frac{1}{2n(n-1)}\sum_{i\neq j}^n (k(\cdot,X_i)-k(\cdot,X_j))\oh(k(\cdot,X_i)-k(\cdot,X_j)),\nonumber
\end{equation*}
conceived from the following alternate representation of $\Sigma$:
$$\Sigma=\frac{1}{2}\int_{\X\times\X}\left(k(\cdot,x)-k(\cdot,y)\right)\oh\left(k(\cdot,x)-k(\cdot,y)\right).$$
Using the reproducing property, it is easy to verify that $$\widehat{\text{Var}}[f(X)]:=\frac{1}{2n(n-1)}\sum_{i\neq j}^n\left(f(X_i)-f(X_j)\right)^2=\inner{\sigh f}{f}_\Hk.$$ Therefore, substituting for $\Sigma$ in (\ref{kpca obj}) yields the objective of empirical KPCA (EKPCA):
\begin{equation}\label{ekpca obj}
    \sup\{\langle f,\sigh f\rangle_\Hk:f\in\Hk,\,\norm{f}_\Hk=1\}.
\end{equation}
Of course $\sigh$ is self-adjoint, positive and has rank at most $n-1$, thus is compact. Thus the spectral theorem \citep{Reed-80} yields
\begin{equation}\label{sigh spectral}
    \sigh=\sum_{i=1}^{n-1}\widehat{\lambda}_i\widehat{\phi}_i\oh\widehat{\phi}_i,
\end{equation}
where $\{\widehat{\lambda}_i\}_{i=1}^{n-1}\subset\R^+$ and $\{\widehat{\phi}_i\}_{i=1}^{n-1}\subset\Hk$ are the eigenvalues and eigenfunctions of $\sigh$. Similar to Assumption \ref{Sigma eigen assum}, we make the following simplifying assumption regarding the spectrum of $\sigh$.
\begin{assumption}\label{sigma_hat eigen assum}
The rank of $\sigh$ is $n-1$, eigenvalues $(\widehat{\lambda}_i)_{i=1}^{n-1}$ of $\sigh$ are simple, positive and w.l.o.g.~satisfy a decreasing rearrangement, i.e., $\widehat{\lambda}_1>\widehat{\lambda}_2>\ldots$
\end{assumption}
For any $\ell\le n-1$, since $\{\phih_i\}_{i=1}^{\ell}$ forms an orthogonal coordinate system in $\Hk$, it yields the following low-dimensional Euclidean representation of $k(\cdot,x)$,
\begin{equation*}
    \left(\inner{k(\cdot,x)}{\phih_1}_\Hk,\cdots,\inner{k(\cdot,x)}{\phih_\ell}_\Hk\right)^\top=\left(\phih_1(x),\cdots,\phih_\ell(x)\right)^\top,\nonumber
\end{equation*}
for any $x\in X$. Moreover, following Assumption \ref{sigma_hat eigen assum}, the orthogonal projector onto $\text{span}\{\phih_i:i=1,\ldots,\ell\}$ is given by
\begin{equation}\label{ekpca projector}
    P^\ell(\sigh)=\sum_{i=1}^\ell\phih_i\oh\phih_i.
\end{equation}
Though $\sigh$ is finite rank, its eigenfunctions are solution to a possibly infinite dimensional linear system. The following result, quoted from \cite{sriperumbudur-20}, shows that the eigenvalues of $\sigh$ can computed by solving an $n$-dimensional system. 
\begin{proposition}[\citealt{sriperumbudur-20}, Proposition 1]\label{prop:ekpca soln}
Let $(\lambdah_i,\phih_i)_i$ be the eigensystem of $\sigh$ in (\ref{sigh spectral}). Define $\mathbf{K}=[k(X_i,X_j)]_{i,j\in[n]}$. Then
$$\phih_i=\frac{1}{\lambdah_i}\sum_{j=1}^n\gamma_{i,j}k(\cdot,X_j),$$
where $\boldsymbol{\gamma}_i=(\gamma_{i,1},\cdots,\gamma_{i,n})=\frac{1}{n(n-1)}\mathbf{H}_n\boldsymbol{\widehat{\alpha}}_i$ with $\boldsymbol{\widehat{\alpha}}_i\notin\emph{null}(\mathbf{H}_n)$, and $(\lambdah_{i},\boldsymbol{\widehat{\alpha}}_i)_i$ are the eigenvalues and eigenvectors of $\frac{1}{n(n-1)}\mathbf{KH}_n$.
\end{proposition}
As computation of the eigensystem of $\sigh$ is obtained by solving an $n\times n$ system, computation of $(\lambdah_i,\phih_i)_{i=1}^\ell$ for $\ell<n$ has a space complexity $O(n^2)$ and a time complexity of $O(n^2\ell)$ via the Lanczos method.

\subsection{Approximate Kernel PCA using the Nystr\"om Method}\label{sec:NY-EKPCA}
For large sample sizes, performing EKPCA amounts to a significant computational burden, motivating many approximation schemes.  We explore the popular Nystr\"om approximation \citep{Reinhardt-85,Williams-01,Drineas-05} to speed up EKPCA and study the trade-offs between  computational gains  and  statistical accuracy. 
The general idea in Nystr\"om method is to obtain a low-rank approximation to the Gram matrix $\mathbf{K}$, and replace $\mathbf{K}$ by this approximation in kernel algorithms, resulting in computational speedup. Since the eigenspace of $\mathbf{KH}_n$ is related to that of $\sigh$ (as noted in Proposition~\ref{prop:ekpca soln}), Nystr\"{o}m method also yields a low rank approximation to $\sigh$, which is what we exploit in developing Nystr\"om approximate KPCA.  It follows from Proposition \ref{prop:ekpca soln} that the eigenfunctions of $\sigh$ lie in the following space,
\begin{eqnarray*}
\bar{\Hk}_n:=\left\{f\in\Hk\Big{|}f=\sum_{i=1}^n\left(n\boldsymbol{\alpha}_i-\sum_{j=1}^n\boldsymbol{\alpha}_j\right)k(\cdot,X_i):\boldsymbol{\alpha}=(\alpha_1,\ldots,\alpha_n)\in\R\right\}.
\nonumber
\end{eqnarray*}
Thus, we could instead express the objective in (\ref{ekpca obj}) as an optimization over $\bar{\Hk}_n$, or equivalently, over $\boldsymbol{\alpha}\in\R^n$.  From this, suppose for some $m<n$ indices $\{r_1,...,r_m\}$ are sampled uniformly without replacement from $[n]$, yielding the subsample $\{X_{r_j}\}_{j=1}^m$ and the random subspace
\begin{eqnarray*}
\bar{\Hk}_m:=\left\{f\in\Hk\Big{|}f=\sum_{j=1}^m\left(m\boldsymbol{\alpha}_j-\sum_{l=1}^m\boldsymbol{\alpha}_l\right)k(\cdot,X_{r_j}):\boldsymbol{\alpha}=(\alpha_1,\ldots,\alpha_m)\in\R\right\}\nonumber
\end{eqnarray*}
of $\Hk$. Nystr\"om KPCA (NY-EKPCA) optimizes the EKPCA objective in (\ref{ekpca obj}) over $\bar{\Hk}_m$, that is, NY-EKPCA is the solution to the following problem:
\begin{equation}\label{NY-EKPCA problem defin}
\sup\left\{\inner{f}{\sigh f}_\Hk:f\in\bar{\Hk}_m,\,\norm{f}_\Hk=1\right\}.
\end{equation}
The following result, whose proof is presented in  Section~\ref{sec:NY-EKPCA soln}, shows that the solution to (\ref{NY-EKPCA problem defin}) is obtained by solving a finite dimensional linear system, which has better computational complexity than that of EKPCA, provided the subsample size is less than the sample size, $m<n$. To this end, we first introduce some notation before stating the result,
\begin{eqnarray*}
    &{}\mathbf{K}_{mm}{}&=[k(X_{r_j},X_{r_l})]_{j,l\in[m]}\in\R^{m\times m},\quad\mathbf{K}_{nm}=[k(X_i,X_{r_j})]_{i\in[n],j\in[m]}\in\R^{n\times m},\quad
    \text{and}\quad\mathbf{K}_{mn}=\mathbf{K}_{nm}^\top.\nonumber
\end{eqnarray*}
\begin{proposition}\label{prop:NY-EKPCA soln}
Define the $m\times m$ matrix $\mathbf{M}=\mathbf{K}_{mm}^{-1/2}\mathbf{K}_{mn}\mathbf{H}_n\mathbf{K}_{nm}\mathbf{K}_{mm}^{-1/2}$.  The solution to (\ref{NY-EKPCA problem defin}) is given by $$\phitild_1=\Tilde{S}_m^*\mathbf{K}_{mm}^{-1/2}\mathbf{u}_1$$
where $\mathbf{u}_1=(u_{1,1},\ldots,u_{1,m})$ is the unit eigenvector of $\frac{1}{n(n-1)}\mathbf{M}$ corresponding to its largest eigenvalue, denoted $\lambdatild_1$ and $\tilde{S}_m^*:\R^m\rightarrow\Hk$, $\boldsymbol{\alpha}\mapsto\sum_{j=1}^m\boldsymbol{\alpha}_jk(\cdot,X_{r_j})$.
\end{proposition}
The complexity of computing $\mathbf{M}$ and its eigendecomposition via the Lanczos method is $O(m^2\ell+nm^2)$; therefore, for $m<n$ the complexity of solving (\ref{NY-EKPCA problem defin}) scales as $O(nm^2)$, which is a reduction from the $O(n^2\ell)$ complexity of solving EKPCA if $m<\sqrt{n\ell}$. It is worth noting the connection between our Nystr\"om approximation to KPCA and the traditional Nystr\"om approximation to the Gram matrix of \cite{Williams-01}, given by 
\begin{equation*}
    \tilde{\mathbf{K}}=\mathbf{K}_{nm}\mathbf{K}_{mm}^{-1}\mathbf{K}_{mn}.\nonumber
\end{equation*}
Observing
\begin{eqnarray*}
    &{}{}&\mathbf{Mu}=\mathbf{K}_{mm}^{-1/2}\mathbf{K}_{mn}\mathbf{H}_n\mathbf{K}_{nm}\mathbf{K}_{mm}^{-1/2}\mathbf{u}=n(n-1)\lambdatild\mathbf{u}\nonumber
    \\
    &{}\implies{}&\Tilde{\mathbf{K}}\mathbf{H}_n\mathbf{K}_{nm}\mathbf{K}_{mm}^{-1/2}\mathbf{u}=n(n-1)\lambdatild\mathbf{K}_{nm}\mathbf{K}_{mm}^{-1/2}\mathbf{u},\nonumber
\end{eqnarray*}
it is clear that $\tilde{\mathbf{K}}\mathbf{H}_n$ will have the same eigenvalues as $\mathbf{M}$. All eigenvalues of $\tilde{\mathbf{K}}\mathbf{H}_n$ and $\mathbf{M}$ will be positive as 
\begin{eqnarray*}
\mathbf{u}^\top\mathbf{M}\mathbf{u}&{}={}&\inner{\mathbf{K}_{mm}^{-1/2}\mathbf{K}_{mn}\mathbf{H}_n\mathbf{K}_{nm}\mathbf{K}_{mm}^{-1/2}\mathbf{u}}{\mathbf{u}}_2=n(n-1)\inner{\sigh\Tilde{S}_m^*\mathbf{K}_{mm}^{-1/2}\mathbf{u}}{\Tilde{S}_m^*\mathbf{K}_{mm}^{-1/2}\mathbf{u}}_\Hk\ge0.\nonumber
\end{eqnarray*}
We will make the following simplifying assumption on $\tilde{\mathbf{K}}\mathbf{H}_n$ and its eigenvalues:
\begin{assumption}\label{cent nystrom eigen assum}
The rank of $\Tilde{\mathbf{K}}\mathbf{H}_n$ is $m$.  The eigenvalues $(\lambdatild_i)_{i=1}^{m-1}$ of $\frac{1}{n(n-1)}\Tilde{\mathbf{K}}\mathbf{H}_n$ are simple, positive, and w.l.o.g. satisfy a decreasing rearrangement, i.e., $\lambdatild_1>\lambdatild_2>\ldots$. 
\end{assumption}
As shown in the proof of Proposition \ref{prop:NY-EKPCA soln} (see Section \ref{sec:NY-EKPCA soln}), $(\phitild_i)_i$ form an orthonormal system.  Thus, for some $\ell<m$ the orthogonal projector onto $\text{span}\{(\phitild_i)_{i=1}^\ell\}$ is given by
\begin{equation}\label{nystrom ekpca projector}
    P_{nys}^\ell(\sigh)=\sum_{i=1}^\ell\phitild_i\oh\phitild_i.
\end{equation}
One may ask if $\phitild_1$ is the eigenfunction of some operator.  Denoting $\bar{P}_m$ as the orthogonal projector onto $\bar{\Hk}_m$, it can be shown (see Section \ref{sec:P_m bar}) that
\begin{equation}\label{eq:P_m bar}
    \bar{P}_m=\Tilde{S}_m^*\mathbf{H}_m(\mathbf{H}_m\mathbf{K}_{mm}\mathbf{H}_m)^+\mathbf{H}_m\Tilde{S}_m,
\end{equation} 
and that $(\phitild_i)_i$ are the orthonormal eigenfunctions of $\bar{P}_m\sigh \bar{P}_m$ with corresponding eigenvalues $(\lambdatild_i)_i$, that is,
\begin{equation}\label{eq:P_m eigen}
    \bar{P}_m\sigh \bar{P}_m\phitild_i=\lambdatild_i\phitild_i.
\end{equation}
Therefore, we may think of $\bar{P}_m\sigh \bar{P}_m$ as a low-rank approximation to $\sigh$.

\section{Computational vs. Statistical Trade-off: Main Results}\label{sec:NY-EKPCA main results}
As shown previously, Nystr\"om kernel PCA approximates the solution to empirical kernel PCA with less computational expense. In this section, we explore whether this computational saving is obtained at the expense of statistical performance. As in \citet{sriperumbudur-20}, we measure the statistical performance of KPCA, EKPCA, and NY-EKPCA in terms of reconstruction error, which is detailed below. 

\subsection{Reconstruction error in $\Hk$-norm}
In linear PCA, the reconstruction error, given by
\begin{equation*}
    \bb{E}_{X\sim\Pb}\norm{(X-\mu)-\sum_{i=1}^\ell\inner{X-\mu}{\phi_i}_2\phi_i}_2^2=\bb{E}_{X\sim\Pb}\norm{(X-\mu)-P^\ell(\Sigma)(X-\mu)}_2^2,\nonumber
    \end{equation*}
is the error involved in reconstructing a centered random variable $X$ by projecting it onto the $\ell$-eigenspace (i.e., span of the top-$\ell$ eigenvectors) associated with its covariance matrix, $\Sigma=\bb{E}[XX^\top]-\E[X]\E[X]^\top$ through the orthogonal projection operator $P^\ell(\Sigma):=\sum^\ell_{i=1}\phi_i\otimes_2 \phi_i$. Clearly, the error is zero when $\ell=d$. The analogs of the reconstruction error in KPCA, EKPCA and NY-EKPCA, can be similarly stated in terms of their projection operators, (\ref{kpca projector}), (\ref{ekpca projector}), and (\ref{nystrom ekpca projector}) as follows:
\begin{eqnarray}\label{cent rec err}
    R^\ell(\Sigma)&{}={}&\E\norm{(k(\cdot,X)-m_\Pb)-P^\ell(\Sigma)(k(\cdot,X)-m_\Pb)}_\Hk^2,\nonumber 
    \\
    R^\ell(\sigh)&{}={}&\E\norm{(k(\cdot,X)-m_\Pb)-P^\ell(\sigh)(k(\cdot,X)-\widehat{m}_\Pb)}_\Hk^2,\label{cent rec err:ekpca}
    \\
     R_{nys}^\ell(\sigh)&{}={}&\E\norm{(k(\cdot,X)-m_\Pb)-P^\ell_{nys}(\sigh)(k(\cdot,X)-\widehat{m}_\Pb)}_\Hk^2\label{cent rec err:nykpca}.
\end{eqnarray}
In (\ref{cent rec err:ekpca}) and (\ref{cent rec err:nykpca}) we center with $\widehat{m}_\Pb:=\frac{1}{n}\sum^n_{i=1}k(\cdot,X_i)$ to ensure that the low-dimensional representation of the reconstructed random variable is computable in practice, as the mean element $m_\Pb$ is likely unknown.  Throughout the rest of the analysis we will drop the $X\sim\Pb$ subscript, and assume expectations are with respect to $X\sim\Pb$ unless otherwise noted.  The following theorem, proved in Section \ref{sec:nys main theorem proof}, provides a finite-sample bound on the reconstruction error associated with NY-EKPCA, under uniform sampling, 
as well as a new result for centered EKPCA, from which convergence rates may be obtained.
\begin{theorem}\label{thm:nys main theorem}
Suppose Assumptions \ref{first kern assum}-\ref{cent nystrom eigen assum} hold.  For any $t>0$ define $\Cal{N}_\Sigma(t)=\emph{tr}(\Sigma(\Sigma+tI)^{-1})$ and $\Cal{N}_{C,\infty}(t)=\sup_{x\in\X}\inner{k(\cdot,x)}{(C+tI)^{-1}k(\cdot,x)}_\Hk$ for the uncentered covariance operator $C=\int_\X k(\cdot,x)\oh k(\cdot,x)d\Pb(x)$. Then the following hold:
\begin{itemize}
    \item[(i)] 
    \begin{equation}
        R^\ell(\Sigma)=\sum_{i>\ell}\lambda_i.\nonumber
    \end{equation}
    \item[(ii)] For any $0\le\delta\le\frac{1}{2}$ satisfying $n\ge2\log\frac{2}{\delta}$ and $\frac{140\kappa}{n}\log\frac{16\kappa n}{\delta}\le t\le\norm{\Sigma}_\OPH$,
    $$\Pb^n\left\{\sum_{i>\ell}\lambda_i\le R^\ell(\sigh)\le3\Cal{N}_\Sigma(t)(\lambda_{\ell+1}+t)+\frac{32\kappa\log\frac{2}{\delta}}{n}\right\}\ge1-5\delta.$$
    \item[(iii)] For any $0\le\delta\le\frac{1}{2}$,
    $$\Pb^n\left\{\sum_{i>\ell}\lambda_i\le R_{nys}^\ell(\sigh)\le6\Cal{N}_\Sigma(t)(\lambda_{\ell+1}+9t)+\frac{32\kappa\log\frac{2}{\delta}}{n}\right\}\ge 1-11\delta,$$
    provided the following conditions are satisfied:
    \begin{itemize}
        \item[1.] $\left(\frac{140\kappa}{n}\log\frac{16\kappa n}{\delta}\vee\frac{9\kappa}{n}\log\frac{n}{\delta}\right)\le t\le \norm{\Sigma}_\OPH\wedge\norm{C}_\OPH$, 
        \item[2.] $m\ge\left(67\vee5\Cal{N}_{C,\infty}(t)\right)\log\frac{4\kappa}{t\delta}\vee\frac{140\kappa}{t}\log\frac{8}{t\delta}$,
        \item[3.] $n\ge2\log\frac{2}{\delta}$.
    \end{itemize}
\end{itemize}
\end{theorem}
\begin{remark} Since $\Sigma$ is trace-class and $\lambda_\ell\rightarrow 0$ as $\ell\rightarrow \infty$, it follows that $R^\ell(\Sigma)\rightarrow 0$ as $\ell\rightarrow \infty$. The rate of this convergence may be analyzed after making assumptions on the decay rate of the $(\lambda_i)_i$, which will be presented in the upcoming corollaries. The behavior of the empirical variations depends significantly on $t$ and $\Cal{N}_{\Sigma}(t)$. $\Cal{N}_\Sigma(t)$ is referred to as the effective dimension or degrees of freedom \citep{Caponnetto-07}, which measures the capacity of the hypothesis space $\Hk$. Upon making assumptions regarding the decay rate of $(\lambda_i)_i$, the size of $\Cal{N}_\Sigma(t)$ can be quantified and convergence rates for $R^\ell(\sigh)$ and $R_{nys}^\ell(\sigh)$ can be obtained. The upper bounds for $R^\ell(\sigh)$ and $R_{nys}^\ell(\sigh)$ are equivalent up to constants; however, the conditions imposed on $m$ and $t$ in (\emph{iii}) will dictate whether this behavior of $R_{nys}^\ell(\sigh)$ may be achieved with a reduced computational complexity ($m<n$). We also would like to highlight that the results presented in Theorem~\ref{thm:nys main theorem}, which are obtained for the U-statistic estimator $\sigh$ of the centered covariance operator, $\Sigma$, matches up to constants, the results in Theorem 2 of \citet{Sterge-20}, which were derived for the uncentered covariance operator, $C$. 
The following corollary derives convergence rates from the bounds in Theorem \ref{thm:nys main theorem} under polynomial decay assumptions on the eigenvalues of $\Sigma$.
\end{remark}
\begin{corollary}[Polynomial decay of eigenvalues]\label{poly decay corollary}
Suppose $\underbar{A}i^{-\alpha}\le\lambda_i\le\bar{A}i^{-\alpha}$ for some $\alpha>1$ and $0<\underbar{A}<\bar{A}<\infty$. Let $\ell=n^{\frac{\theta}{\alpha}}$, where $0<\theta\le \alpha$. Then the following hold:
\begin{itemize}
    \item[(i)] $$n^{-\theta(1-\frac{1}{\alpha})}\lesssim R^\ell(\Sigma)\lesssim n^{-\theta(1-\frac{1}{\alpha})}.$$
    \noindent There exists an $N\in\bb{N}$ such that for all $n>N$, the following hold:
    \item[(ii)] 
    \begin{equation*}
        n^{-\theta(1-\frac{1}{\alpha})}\lesssim R^\ell(\sigh)\lesssim_{\mathbb{P}^n}
        \begin{cases}
            n^{-\theta(1-\frac{1}{\alpha})},\quad\quad{}\theta<1{} \\ 
            \left(\frac{\log n}{n}\right)^{1-\frac{1}{\alpha}},\quad {}\theta\ge1{}
        \end{cases};
    \end{equation*}
    \item[(iii)]
    \begin{equation*}
        n^{-\theta(1-\frac{1}{\alpha})}\lesssim R_{nys}^\ell(\sigh)\lesssim_{\mathbb{P}^n}
        \begin{cases}
           n^{-\theta(1-\frac{1}{\alpha})},\quad\quad{}\theta<1,\,m\gtrsim n^\theta\log n{} \\ 
            \left(\frac{\log n}{n}\right)^{1-\frac{1}{\alpha}},\quad{}\theta\ge1\,,m\gtrsim\frac{n}{\log n}\log\frac{n}{\log n}{}
        \end{cases}.
    \end{equation*}
\end{itemize}
\end{corollary}
\begin{remark}\label{remark:post poly decay corollary}
\emph{(i)} Of course, $\alpha>1$ is required to ensure that $\Sigma$ is trace class.  Observing (\emph{ii}) and (\emph{iii}), we see that the convergence rates of $R^\ell(\sigh)$ and $R_{nys}^\ell(\sigh)$ rely heavily on the growth of $\ell$ through $\theta$. Comparing $R^\ell(\sigh)$ to $R^\ell(\Sigma)$, EKPCA will match the convergence rate of KPCA provided $\ell$ does not grow faster than $n^{1/\alpha}$. We note that $0<\theta<1$ is the only sensible regime both computationally and statistically, as $\theta\ge1$ increases the computational complexity while the rate plateaus at $\left(\log n/n\right)^{1-\frac{1}{\alpha}}$.
\par\noindent\emph{(ii)} When $\theta<1$, the convergence rate of $R_{nys}^\ell(\sigh)$ is equal to that of $R^\ell(\Sigma)$ and $R^\ell(\sigh)$, provided $m\gtrsim n^\theta\log n$, i.e., if $\ell$ grows to infinity not faster than $n^{1/\alpha}$ and the number of subsamples $m$ grows sufficiently fast, then NY-EKPCA and EKPCA enjoy the same statistical behavior. From a computational perspective, the computational complexity of EKPCA using the Lanczos method is $O(n^{2+\frac{\theta}{\alpha}})$, while the complexity of NY-EKPCA is $O(nm^2+m^2\ell)=O(nm^2)$. Thus, NY-EKPCA will offer a computational advantage with no loss in statistical performance, if  $\theta<\frac{1}{2}+\frac{\theta}{2\alpha}$, i.e., $\theta<\frac{\alpha}{2\alpha-1}$. This means NY-EKPCA has better computational complexity and same statistical rates for $\theta<\frac{\alpha}{2\alpha-1}$ while it loses the computational edge with no loss in the statistical behavior when $ \frac{\alpha}{2\alpha-1}\le\theta<1$. Note that the first few principal components are often the greatest interest in practice; thus, the case $\theta<1$ may be more relevant in application. 
\end{remark}

\subsection{Reconstruction error in $\lp$-norm}\label{sec:main results lp}
While defining the reconstruction error of NY-KPCA in $\Hk$ is the most natural construction, we additionally study the reconstruction error of NY-KPCA in $\lp$ to allow for comparison with another popular approximation based on random features. Random feature approximation \citep{Rahimi-08a} computes a random low-dimensional approximation of the kernel, which may be used in place of $k$ in learning methodologies to achieve reduced computational complexity. To elaborate, we define the \textit{feature map} $\Phi(x):=k(\cdot,x)$ and consider kernels of the form
$$k(x,y)=\int_\Theta\varphi(x,\theta)\varphi(y,\theta)d\Lambda(\theta),$$
where $\varphi(x,\cdot)\in L^2(\Theta,\Lambda)$ for all $x\in\X$ and $\Lambda$ is a probability measure (w.l.o.g.) on a measurable space $\Theta$. For a random sample $(\theta_i)_{i=1}^m\stackrel{\iid}{\sim}\Lambda$, the random feature approximation to $k$ is constructed as
\begin{equation*}
    k_m(x,y)=\frac{1}{m}\sum_{i=1}^m\varphi(x,\theta_i)\varphi(y,\theta_i)=\inner{\Phi_m(x)}{\Phi_m(y)}_2,\nonumber
\end{equation*}
where
$\Phi_m(x)=\frac{1}{\sqrt{m}}\left(\varphi(x,\theta_1),\ldots,\varphi(x,\theta_m))\right)^\top$
is the \textit{approximate feature map}. It can be shown (\citealt[Section 3.3]{sriperumbudur-20}) that $k_m$ is the reproducing kernel of an $m$-dimensional RKHS, denoted $\Hk_m$, which is isometrically isomorphic to $\R^m$. Random feature KPCA (RF-KPCA) involves solving 
$$\arg\sup\left\{\left\langle f,\Sigma_m f\right\rangle_{\mathcal{H}_m}:f\in\mathcal{H}_m,\Vert f\Vert_{\mathcal{H}_m}=1\right\}$$
where $$\Sigma_m=\int_\X k_m(\cdot,x)\ohm k_m(\cdot,x)d\Pb(x)-\left(\int_\X k_m(\cdot,x)d\Pb(x)\right)\ohm\left(\int_\X k_m(\cdot,x)d\Pb(x)\right),$$
is the \textit{approximate} covariance operator on $\mathcal{H}_m$ induced by $k_m$. Note that RF-KPCA is exactly KPCA but with $\Sigma$ and $\mathcal{H}$ being replaced by their approximate counterparts, i.e., $\Sigma_m$ and $\mathcal{H}_m$, respectively. This means, the solution to RF-KPCA is the eigenfunction that corresponds to the maximum eigenvalue of $\Sigma_m$. The empirical version of RF-KPCA, referred to as RF-EKPCA, involves solving RF-KPCA with $\Sigma$ replaced by its U-statistic estimator
$$\sigh_m=\frac{1}{2n(n-1)}\sum_{i\neq j}^n\left(k_m(\cdot,X_i)-k_m(\cdot,X_j)\right)\ohm\left(k_m(\cdot,X_i)-k_m(\cdot,X_j)\right).$$ Note that the computation of top-$\ell$ eigenfunctions, $(\phih_{m,i})_{i=1}^\ell\subset\Hk_m$, of $\sigh_m$ by RF-EKPCA has  a computational complexity $O(m^2\ell+m^2n)$, an improvement upon EKPCA provided $m<\sqrt{n\ell}$. The orthogonal projection operator $P^\ell_{rf}(\sigh_m)=\sum_{i=1}^\ell\phih_{m,i}\ohm\phih_{m,i}$ may then be used to compute a low-dimensional representation of $X_i\in\X$. We emphasize that in Nystr\"om approximation, $m$ is the number of subsampled indices, while in random features, $m$ is the number of random features sampled.

Unlike NY-EKPCA, where the eigenfunctions reside in $\Cal{H}$, the eigenfunctions $(\phih_{m,i})_i$ lie in $\Hk_m$. Therefore, the reconstruction error of RF-EKPCA in $\Hk$-norm is ill-defined. To remedy this issue, \citet{sriperumbudur-20} consider two different versions of reconstruction error, named \textit{Reconstruct and Embed} (R-E) and \textit{Embed and Reconstruct} (E-R), in which elements in $\Hk$ and $\Hk_m$ are mapped to a common space, $\lp$ through the inclusion and approximation operators,
$$\id:\Hk\rightarrow\lp,\,f\mapsto f-f_\bb{P},$$ and $$\mathfrak{A}:\Cal{H}_m\rightarrow \lp,\quad f=\sum^m_{i=1}\beta_i\varphi_i\mapsto \sum^m_{i=1}\beta_i(\vp_i-\vp_{i,\bb{P}})=f-f_\bb{P}$$
where $\vp_{i,\bb{P}}:=\intx \vp_i(x)\,d\bb{P}(x)$, $\varphi_i:=\varphi(\cdot,\theta_i)$ and $f_\bb{P}=\int_\X f(x)d\Pb(x)$. As the names suggest, in R-E, the functions are first reconstructed in $\Hk$ based on principal components and then embedded into $\lp$, while in E-R, the functions are first embedded into $\lp$ and then reconstructed in $\lp$ based on the embedded principal components. 
Based on the observation that $(\phi_i)_i$ and $\left(\frac{\id\phi_i}{\sqrt{\lambda}_i}\right)_i$ form orthonormal systems in $\Hk$ and $\lp$ respectively, R-E and E-R correspond to the following reconstruction errors for KPCA in $\lp$: 
\begin{equation*}
    T^\ell(\Sigma)=\E\norm{\id\kbar-\id\sum_{i=1}^\ell\inner{k(\cdot,X)-m_\Pb}{\phi_i}_\Hk\phi_i}_\lp^2\nonumber
\end{equation*}
and
\begin{equation*}
    S^\ell(\Sigma)=\E\norm{\id\kbar-\sum_{i=1}^\ell\inner{\id(k(\cdot,X)-m_\Pb)}{\frac{\id\phi_i}{\sqrt{\lambda}_i}}_\lp\frac{\id\phi_i}{\sqrt{\lambda}_i}}_\lp^2\nonumber
\end{equation*}
Theorems 2 and 5 of \cite{sriperumbudur-20} show these definitions of reconstruction error to be equivalent for KPCA, i.e.,
$T^\ell(\Sigma)=S^\ell(\Sigma)=\sum_{i>\ell}\lambda_i^2,$ while it is not the case for EKPCA and RF-EKPCA. We refer the reader to \cite{sriperumbudur-20} for more details about the behavior of these reconstruction errors for EKPCA and RF-EKPCA.

In this paper, we analyze NY-EKPCA in $L^2(\mathbb{P})$ norm w.r.t.~the reconstruction error of R-E as the E-R reconstruction error is less user-friendly with the principal components being non-computable because of their dependence on $\mathbb{P}$ through $\id$.
Therefore, the reconstruction error in $L^2(\mathbb{P})$ for EKPCA, RF-EKPCA and NY-EKPCA are given by
\begin{eqnarray*}
     T^\ell(\sigh)&{}={}&\E\norm{\id\kbar-\id P^\ell(\sigh)(k(\cdot,X)-\widehat{m}_\Pb)}_\lp^2,\nonumber
    \\
    T^\ell_{rf}(\sigh_m)&{}={}&\E\norm{\id\kbar-\mathfrak{A} P^\ell_{rf}(\sigh_m)(k_m(\cdot,X)-\widehat{m}_{\Pb,m})}_\lp^2,\nonumber
    \\
    T_{nys}^\ell(\sigh)&{}={}&\E\norm{\id\kbar-\id P^\ell_{nys}(\sigh)(k(\cdot,X)-\widehat{m}_\Pb)}_\lp^2,\nonumber
\end{eqnarray*}
respectively, where $\widehat{m}_{\Pb,m}=\frac{1}{n}\sum^n_{i=1}k_m(\cdot,X_i)$. Of course, $\norm{\cdot}_\lp$ is weaker than $\norm{\cdot}_\Hk$, so naturally we can expect better error behavior of $T_{nys}^\ell(\sigh)$ when compared with $R_{nys}^\ell(\sigh)$. However, the interesting comparison is between $T_{nys}^\ell(\sigh)$ and $T_{rf}^\ell(\sigh)$.  To facilitate this comparison, we first present the following result, which gives finite sample bounds on $T^\ell(\Sigma),\,T^\ell(\sigh)$ and $T_{rf}^\ell(\sigh_m)$.
\begin{theorem}[\citealt{sriperumbudur-20}, Theorem 2]\label{rf-kpca main thm}
Suppose Assumptions \ref{first kern assum}-\ref{cent nystrom eigen assum} hold.  For any $t>0$ define $\Cal{N}_\Sigma(t)=\emph{tr}(\Sigma(\Sigma+tI)^{-1})$. Then the following hold:
\begin{itemize}
    \item[(i)]
    \begin{equation*}
        T^\ell(\Sigma)=\sum_{i>\ell}\lambda_i^2.
    \end{equation*}
    \item[(ii)] For any $\delta>0$ with $n\ge2\log\frac{2}{\delta}$ and $\frac{140\kappa}{n}\log\frac{16\kappa n}{\delta}\le t\le\norm{\Sigma}_\OPH$,
    \begin{equation*}
        \Pb^n\left\{\sum_{i\ge\ell}\lambda_i^2\le T^\ell(\sigh)\le9\Cal{N}_\Sigma(t)(\lambda_{\ell+1}+t)^2+\frac{64\kappa^2\log\frac{2}{\delta}}{n}\right\}\ge1-3\delta.
    \end{equation*}
    \item[(iii)] Suppose $m$ random features are sampled $\iid$ from a probability measure $\Lambda$. For any $\delta>0$ with $n\ge 2\log\frac{2}{\delta}$, $m\ge \left(2\vee \frac{1024\kappa^2}{\sum_{i>\ell}\lambda^2_i}\right)\log\frac{2}{\delta}$ and $\frac{140\kappa}{n}\log\frac{16\kappa n}{\delta}\vee\frac{86\kappa}{m}\log\frac{16\kappa m}{\delta}\le t\le\frac{\norm{\Sigma}_{\OPH}}{3}$, with probability at least $1-12\delta$ over the joint measure $\Pb^n\times\Lambda^m$:
    \begin{equation*}
    \frac{1}{4}\sum_{i>\ell}\lambda^2_{i}\le T_{rf}^\ell(\sigh_m)\le162\Cal{A}_1(t) (\lambda_{\ell+1}+t)^2+\frac{640\kappa^2\log\frac{2}{\delta}}{3n}+\frac{256\kappa^2\log\frac{2}{\delta}}{m},\nonumber
\end{equation*}
where $\Cal{A}_1(t):=\Cal{N}_\Sigma(t)+\frac{16\kappa\log\frac{2}{\delta}}{tm}+\sqrt{\frac{8\kappa\Cal{N}_\Sigma(t)\log\frac{2}{\delta}}{tm}}$.
\end{itemize}
\end{theorem}
We must note that unlike \emph{(i)-(ii)}, the probability statement in \emph{(iii)} is with respect to a joint measure, as it must consider the randomness introduced by the random features. In light of the previously presented result, we now present an analogous result for NY-EKPCA in $\lp$-norm.
\begin{theorem}\label{thm:nys main theorem l2}
Under the same assumptions as in Theorem \ref{thm:nys main theorem}, 
$$\Pb^n\left\{\sum_{i>\ell}\lambda_i^2\le T_{nys}^\ell(\sigh)\le36\Cal{N}_\Sigma(t)(\lambda_{\ell+1}+9t)^2+\frac{32\kappa^2\log\frac{2}{\delta}}{n}\right\}\ge1-11\delta,$$
provided the following conditions are satisfied:
    \begin{itemize}
        \item[1.] $\left(\frac{140\kappa}{n}\log\frac{16\kappa n}{\delta}\vee\frac{9\kappa}{n}\log\frac{n}{\delta}\right)\le t\le \norm{\Sigma}_\OPH\wedge\norm{C}_\OPH$,
        \item[2.] $m\ge\left(67\vee5\Cal{N}_{C,\infty}(t)\right)\log\frac{4\kappa}{t\delta}\vee\frac{140\kappa}{t}\log\frac{8}{t\delta}$,
        \item[3.] $n\ge2\log\frac{2}{\delta}$.
    \end{itemize}
\end{theorem}
\begin{remark}
\emph{(i)} Comparing Theorem \ref{thm:nys main theorem l2} with Theorem~\ref{rf-kpca main thm}\emph{(ii)}, we note that the bounds for EKPCA and NY-EKPCA are identical up to constants, similar to the case in Theorem \ref{thm:nys main theorem}. Compared to their counterparts in  Theorem \ref{thm:nys main theorem}, $T^\ell(\sigh)$ and $T_{nys}^\ell(\sigh)$ have similar dependence on the effective dimension, but a squared dependence on $\lambda_{\ell+1}$ and $t$---in contrast to a linear dependence in Theorem \ref{thm:nys main theorem}. This is a byproduct of working with the $\lp$-norm, which is weaker than the RKHS norm, and therefore will result in faster convergence rates, as will be evident in Corollary \ref{poly decay corollary lp}. Additionally, $T^\ell(\Sigma)$ will decay as $\ell\rightarrow\infty$ more rapidly than $R^\ell(\Sigma)$, as it depends on the sum of squared eigenvalues. The error in estimating the mean element is not improved in the move to $\lp$-norm; it is bounded as $n^{-1}$ in all of the empirical varieties regardless of norm. 
\vspace{1mm}\\
\emph{(ii)} An immediate difference between NY-EKPCA and RF-EKPCA is the dependence on $m$. This difference can be seen in both the upper bounds of $T_{rf}^\ell(\sigh_m)$ and $T_{nys}^\ell(\sigh)$, as well as the size requirements on $m$. This is primarily due to the approximation error incurred by RF-EKPCA, which approximates $\Hk$ with an $m$-dimensional RKHS. Of course, this dependence on $m$ is crucial in analyzing the computational vs.~statistical trade-off between the two methods.
\end{remark}
\begin{corollary}\label{poly decay corollary lp}
Under the same assumptions as in Corollary \ref{poly decay corollary}, the following hold:
\begin{itemize}
    \item[(i)] $$n^{-2\theta(1-\frac{1}{2\alpha})}\lesssim T^\ell(\Sigma)\lesssim n^{-2\theta(1-\frac{1}{2\alpha})}.$$
    There exists an $N\in\bb{N}$ such that for all $n>N$, the following hold:
    \item[(ii)]
    \begin{equation*}
        n^{-2\theta(1-\frac{1}{2\alpha})}\lesssim T^\ell(\sigh)\lesssim_{\Pb^n}
        \begin{cases}
        n^{-2\theta(1-\frac{1}{2\alpha})},\quad{}\theta<\frac{\alpha}{2\alpha-1}{} \\
        \frac{1}{n},\,\,\,\,\,\quad\quad\quad\quad{}\theta\ge\frac{\alpha}{2\alpha-1}{}
        \end{cases};
    \end{equation*}
    \item[(iii)]
    \begin{equation*}
        n^{-2\theta(1-\frac{1}{2\alpha})}\lesssim T_{nys}^\ell(\sigh)\lesssim_{\Pb^n}
        \begin{cases}
            n^{-2\theta(1-\frac{1}{2\alpha})},\quad{}\theta<\frac{\alpha}{2\alpha-1},\,m\gtrsim n^\theta\log n{} \\ \\
            \frac{1}{n},\quad\quad\quad\quad\,\,\,\,\,{}\theta\ge\frac{\alpha}{2\alpha-1}\,,m\gtrsim n^\frac{\alpha}{2\alpha-1}\log n{}
        \end{cases};
    \end{equation*}
    \item[(iv)] 
    \begin{equation*}
    n^{-2\theta(1-\frac{1}{2\alpha})}\mathbf{1}_{\{\gamma\ge\theta(2-\frac{1}{\alpha})\}}\lesssim T_{rf}^\ell({\sigh_m})\lesssim_{\Pb^n\times\Lambda^m}
        \begin{cases}
        n^{-2\theta(1-\frac{1}{2\alpha})},\quad{}\gamma\ge\theta(2-\frac{1}{\alpha}),\,\theta<\frac{\alpha}{2\alpha-1}{} \\ \\
        n^{-\gamma},\quad\quad\quad\quad{}\gamma<1\wedge\theta(2-\frac{1}{\alpha}){}
        \end{cases}
    \end{equation*}
    where $m=n^\gamma$ for $0<\gamma\le1$.
\end{itemize}
\end{corollary}
\begin{remark}Results \emph{(i), (ii)}, and \emph{(iv)} are quoted from \citet[Corollary 3]{sriperumbudur-20}.\vspace{1mm}\\
\emph{(i)} We first highlight the difference between the $\lp$ and $\Hk$ norms through the comparison of Corollaries \ref{poly decay corollary} and \ref{poly decay corollary lp}.  $T^\ell(\Sigma)$ decays at a rate of $n^{-2\theta(1-\frac{1}{2\alpha})}$, which is faster than that of its analog in $\Hk-$norm, i.e., $R^\ell(\Sigma)$. While $R^\ell(\sigh)$ and $R_{nys}^\ell(\sigh)$ recover the optimal convergence rate (compared to KPCA) in the range $\theta<1$, $T^\ell(\sigh)$ and $T_{nys}^\ell(\sigh)$ are only able to recover the optimal rate for $\theta<\frac{\alpha}{2\alpha-1}$. 
The $\frac{1}{n}$ term, which arises due to the empirical recentering, is never dominant in $R^\ell(\sigh)$ and $R_{nys}^\ell(\sigh)$; however, it can dominate in $T^\ell(\sigh)$ and $T_{nys}^\ell(\sigh)$ depending on the range of $\theta$.\vspace{1mm}\\
\noindent\emph{(ii)} Observing $T^\ell(\sigh)$ and $T_{nys}^\ell(\sigh)$ we see that, as in Corollary \ref{poly decay corollary}, NY-EKPCA and EKPCA have similar convergence behavior, provided $m$ is large enough.  
Therefore, it follows from Remark \ref{remark:post poly decay corollary}, that in both $\Hk$ and $\lp$, NY-EKPCA will provide less computational cost with no loss in statistical performance compared to EKPCA. However in $\lp$, unlike in $\Hk$, NY-EKPCA is computationally advantageous than EKPCA regardless of the size of $\theta$.
\vspace{1mm}\\
\noindent\emph{(iii)} When $\theta<\frac{\alpha}{2\alpha-1}$, both RF-EKPCA and NY-EKPCA achieve the optimal convergence rate of $n^{-2\theta(1-\frac{1}{2\alpha})}$, but with NY-EKPCA being more computationally efficient than RF-EKPCA. This is because, in this regime, RF-EKPCA scales as $O(n^{1+2\gamma})$ and NY-EKPCA as $O(n^{1+2\theta})$ with $\gamma\ge\theta(2-\frac{1}{\alpha})>\theta$. In the range $\theta\ge\frac{\alpha}{2\alpha-1}$, NY-EKPCA achieves the optimal convergence rate of $\frac{1}{n}$, while RF-EKPCA converges as $n^{-\gamma}$ with $\gamma<1$. Further, in this range of $\theta$, NY-EKPCA will offer less computational cost for $\gamma\ge\frac{\alpha}{2\alpha-1}$, as RF-EKPCA scales as $O(n^{1+2\gamma})$ and NY-EKPCA as $O(n^{1+\frac{2\alpha}{2\alpha-1}})$. 
\end{remark}

\section{Discussion}
We have studied the trade-off between statistical accuracy and computational complexity in approximate kernel PCA using the Nystr\"om method. While it is clear that Nystr\"om kernel PCA will offer a computational advantage for Nystr\"om subsamples $m<n$, we showed the error in reconstructing $k(\cdot,X)$ in $\Hk$ using $\ell$-eigenfunctions in Nystr\"om kernel PCA to be statistically optimal when compared to standard kernel PCA, provided $m$ is large enough, but still $m<n$, and $\ell$ small enough. Additionally, the size of $m$ depends on the number of eigenfunctions $\ell$; larger $\ell$ requires more subsamples to achieve the best possible statistical behavior. Additionally, unlike several existing theoretical works on kernel PCA, we derived these results by not assuming the mean element of $k$ to be zero. Further, we adapted our notion of reconstruction error to the $\lp$ setting in order to compare Nystr\"om kernel PCA with random feature-based kernel PCA. In $\lp$, we showed  Nystr\"om kernel PCA to achieve the best possible statistical behavior, while maintaining its computational edge regardless of the number of eigenfucntions $\ell$. In comparison to random features, we showed that the reconstruction error of Nystr\"om KPCA converges to zero faster than random features with less computational complexity. As mentioned in Section \ref{sec:main results lp}, further comparison to random features could be explored by considering an alternative definition of reconstruction error in $\lp$.

\par While this work considers only plain Nystr\"om, that is, choosing subsamples uniformly, it is possible to choose these subsamples with probabilities proportional to their individual leverage scores, defined as the diagonal entries of the matrix $(\mathbf{K}+nt\mathbf{I}_n)^{-1}\mathbf{K}$. Consideration of Nystr\"om subsampling according to the leverage scores has yielded success in kernel ridge regression \citep{Rudi-15, Alaoui-15}, where it has led to relaxed requirements on the size of $m$ necessary to achieve best possible statistical behavior. Thus, the computational benefit provided by Nystr\"om is more pronounced when points are sampled according to the leverage score distribution. Though leverage score sampling has been studied successfully in uncentered Nystr\"om KPCA \citep{Sterge-20}, the U-statistic estimator considered in this work introduces significant technical challenges, and so is relegated to future work.

\section{Proofs}
In this section we present the proofs of the main results of the paper.

\subsection{Proof of Proposition \ref{prop:NY-EKPCA soln}}\label{sec:NY-EKPCA soln}
Note that any $f\in\bar{\Hk}_m$ can be written as $\tilde{S}_m^*\mathbf{H}_m\boldsymbol{\alpha}$ for some $\boldsymbol{\alpha}\in\R^m$.  Thus, we may express (\ref{NY-EKPCA problem defin}) as
\begin{eqnarray*}
\sup_{\substack{{\boldsymbol{\alpha}\in\R^m}\\{\norm{\tilde{S}_m^*\mathbf{H}_m\boldsymbol{\alpha}}_\Hk=1}}}\inner{\sigh\tilde{S}_m^*\mathbf{H}_m\boldsymbol{\alpha}}{\tilde{S}_m^*\mathbf{H}_m\boldsymbol{\alpha}}_\Hk=\sup_{\substack{{\boldsymbol{\alpha}\in\R^m}\\{\norm{\mathbf{K}^{1/2}_{mm}\mathbf{H}_m\boldsymbol{\alpha}}_2=1}}}\frac{1}{n(n-1)}\inner{\mathbf{K}_{mn}\mathbf{H}_n\mathbf{K}_{nm}\mathbf{H}_m\boldsymbol{\alpha}}{\mathbf{H}_m\boldsymbol{\alpha}}_2,\nonumber
\end{eqnarray*}
\noindent where we have used Lemma \ref{lem:sampling op}\emph{(v)}.  Let $\mathbf{u}=\mathbf{K}_{mm}^{1/2}\mathbf{H}_m\boldsymbol{\alpha}$, and the above problem may be written as
\begin{equation*}
\sup_{\substack{{\mathbf{u}\in\text{ran}(\mathbf{K}_{mm}^{1/2}\mathbf{H}_m)}\\{\mathbf{u}^\top\mathbf{u}=1}}}\frac{1}{n(n-1)}\inner{\mathbf{K}_{mn}\mathbf{H}_n\mathbf{K}_{nm}\mathbf{K}_{mm}^{-1/2}\mathbf{u}}{\mathbf{K}_{mm}^{-1/2}\mathbf{u}}_2=\sup_{\substack{{\mathbf{u}\in\text{ran}(\mathbf{K}_{mm}^{1/2}\mathbf{H}_m)}\\{\mathbf{u}^\top\mathbf{u}=1}}}\frac{1}{n(n-1)}\mathbf{u}^\top\mathbf{M}\mathbf{u},\nonumber
\end{equation*}
where $\mathbf{M}:=\mathbf{K}_{mm}^{-1/2}\mathbf{K}_{mn}\mathbf{H}_n\mathbf{K}_{nm}\mathbf{K}_{mm}^{-1/2}$. The solution to the above problem is the unit eigenvector of $\frac{1}{n(n-1)}\mathbf{M}$ corresponding to its largest eigenvalue; denote this eigenvector as $\mathbf{u}_1$ with eigenvalue $\lambdatild_1$. We then have $\mathbf{H}_m\boldsymbol{\alpha}_1=\mathbf{K}_{mm}^{-1/2}\mathbf{u}_1$ yielding the function in $\bar{\Hk}_m$,
\begin{equation*}
    \phitild_1=\tilde{S}_m^*\mathbf{K}_{mm}^{-1/2}\mathbf{u}_1,\nonumber
\end{equation*}
\noindent solving (\ref{NY-EKPCA problem defin}).  Subsequent eigenfunctions $\phitild_i$ may be computed in a similar manner from the eigenvectors of $\frac{1}{n(n-1)}\mathbf{M}$, and the orthonormality of the $\{\phitild_i\}_i$ follows from the orthonormality of the $\{\mathbf{u}_i\}_i$, i.e.,
$$\inner{\phitild_i}{\phitild_j}_\Hk=\inner{\tilde{S}_m^*\mathbf{K}_{mm}^{-1/2}\mathbf{u}_i}{\tilde{S}_m^*\mathbf{K}_{mm}^{-1/2}\mathbf{u}_j}_2=\inner{\mathbf{K}_{mm}^{-1/2}\mathbf{K}_{mm}\mathbf{K}_{mm}^{-1/2}\mathbf{u}_i}{\mathbf{u}_j}_2=\delta_{ij}.$$

\subsection{Proofs of (\ref{eq:P_m bar}) and (\ref{eq:P_m eigen})}\label{sec:P_m bar}
(\ref{eq:P_m bar}) is immediate because $\bar{\Hk}_m=\text{ran}(\Tilde{S}_m^*\mathbf{H}_m)$ and $\mathbf{H}_m\Tilde{S}_m\Tilde{S}_m^*\mathbf{H}_m=\mathbf{H}_m\mathbf{K}_{mm}\mathbf{H}_m$. To verify (\ref{eq:P_m eigen}), because $\phitild_i\in\bar{\Hk}_m$ for all $i\in[m]$, we have
\begin{eqnarray*}
    \bar{P}_m\sigh\bar{P}_m\phitild_i&{}={}&\bar{P}_m\sigh\phitild_i=\frac{1}{n(n-1)}\bar{P}_mS_n^*\mathbf{H}_nS_n\Tilde{S}_m^*\mathbf{K}_{mm}^{-1/2}\mathbf{u}_i\nonumber
    \\
    &{}={}&\frac{1}{n(n-1)}\Tilde{S}_m^*\mathbf{H}_m(\mathbf{H}_m\mathbf{K}_{mm}\mathbf{H}_m)^+\mathbf{H}_m\mathbf{K}_{mn}\mathbf{H}_n\mathbf{K}_{nm}\mathbf{K}_{mm}^{-1/2}\mathbf{u}_i\nonumber
    \\
    &{}={}&\frac{1}{n(n-1)}\Tilde{S}_m^*\mathbf{H}_m(\mathbf{H}_m\mathbf{K}_{mm}\mathbf{H}_m)^+\mathbf{H}_m\mathbf{K}_{mm}^{1/2}\mathbf{M}\mathbf{u}_i\nonumber
    \\
    &{}={}&\lambdatild_i\Tilde{S}_m^*\mathbf{H}_m(\mathbf{H}_m\mathbf{K}_{mm}\mathbf{H}_m)^+\mathbf{H}_m\mathbf{K}_{mm}\mathbf{H}_m\boldsymbol{\alpha}_i\nonumber
    \\
    &{}={}&\lambdatild_i\bar{P}_m\tilde{S}_m^*\mathbf{H}_m\boldsymbol{\alpha}_i=\lambdatild_i\phitild_i,\nonumber
\end{eqnarray*}
completing the proof.

\subsection{Proof of Theorem \ref{thm:nys main theorem}}\label{sec:nys main theorem proof}
For notational convenience, we define $\kbar:=k(\cdot,X)-m_\Pb$.\\\\
$(i)$ From Lemma \ref{lem:recons rewrite 1}, we have
$$R^\ell(\Sigma)=\norm{(I-P^\ell(\Sigma))\Sigma^{1/2}}_\HSH^2=\text{tr}\left((I-P^\ell(\Sigma))\Sigma(I-P^\ell(\Sigma))\right)=\sum_{i>\ell}\lambda_i.$$
\\
$(ii)$ \emph{Upper Bound:} We write
\begin{eqnarray}
    R^\ell(\sigh)&{}={}&\E\norm{(k(\cdot,X)-m_\Pb)-P^\ell(\sigh)(k(\cdot,X)-\widehat{m}_\Pb)}_\Hk^2\nonumber
    \\
    &{}={}&\E\norm{(I-P^\ell(\sigh))\kbar}_\Hk^2+\norm{P^\ell(\sigh)(m_\Pb-\widehat{m}_\Pb)}_\Hk^2\nonumber
    \\
    &{}{}&\qquad-2\E\inner{(I-P^\ell(\sigh))\kbar}{P^\ell(\sigh)(m_\Pb-\widehat{m}_\Pb)}_\Hk\nonumber\\
    &{}={}&\underbrace{\norm{(I-P^\ell(\sigh))\Sigma^{1/2}}_\HSH^2}_{\circled{A}}+\underbrace{\norm{P^\ell(\sigh)(m_\Pb-\widehat{m}_\Pb)}_\Hk^2}_{\circled{B}},\label{eq:nys main 2.2}
\end{eqnarray}
\noindent where the last equality holds because $\E[\kbar]=0$ and we have employed Lemma \ref{lem:recons rewrite 1}. For any $t>0$ we have
\begin{eqnarray}
    \circled{A}&{}={}&\norm{(I-P^\ell(\sigh))(\sigh+tI)^{1/2}(\sigh+tI)^{-1/2}(\Sigma+tI)^{1/2}(\Sigma+tI)^{-1/2}\Sigma^{1/2}}_\HSH^2\nonumber
    \\
    &{}\le{}&\norm{(\Sigma+tI)^{-1/2}\Sigma^{1/2}}_\HSH^2\norm{(\sigh+tI)^{-1/2}(\Sigma+tI)^{1/2}}_\OPH^2\norm{(I-P^\ell(\sigh))(\sigh+tI)^{1/2}}_\OPH^2\nonumber
    \\
    &{}\stackrel{(\dag)}{\le}{}&\Cal{N}_\Sigma(t)\norm{(\sigh+tI)^{-1/2}(\Sigma+tI)^{1/2}}_\OPH^2(\widehat{\lambda}_{\ell+1}+t)\label{ekpca A bnd},
\end{eqnarray}
\noindent where we have used
\begin{equation*}
    \norm{(\Sigma+tI)^{-1/2}\Sigma^{1/2}}_\HSH^2=\text{tr}\left(\Sigma^{1/2}(\Sigma+tI)^{-1}\Sigma^{1/2}\right)=\Cal{N}_\Sigma(t),\nonumber%
\end{equation*}
which holds via invariance of trace under cyclic permutations, in $(\dag)$. The result follows from applying Lemma \ref{lem:op bounds} to (\ref{ekpca A bnd}) and Lemma \ref{lem:kernel mean bnd} to \circled{B}, noticing that
$$\circled{B}\le\norm{P^\ell(\sigh)}_\OPH^2\norm{m_\Pb-\widehat{m}_\Pb}_\Hk^2\le\norm{m_\Pb-\widehat{m}_\Pb}_\Hk^2.$$
\\
\noindent \emph{Lower Bound:} It is clear from (\ref{eq:nys main 2.2}) that $R^\ell(\sigh)\ge\circled{A}$.  We will show that 
$$R^\ell(\Sigma)=\inf_{
\{\psi_i\}_i\in Q}\norm{\left(I-P_{\psi,\ell}\right)\Sigma^{1/2}}_\HSH^2$$
\noindent where $P_{\psi,\ell}=\sum_{i=1}^\ell\psi_i\oh\psi_i$ and $Q=\left\{\{\psi_i\}_{i=1}^\ell\subset\Hk:\inner{\psi_i}{\psi_j   }_\Hk=\delta_{ij},\forall i,j\in[\ell]\right\}$, which in turn implies that $\circled{A}\ge R^\ell(\Sigma)$. We have
\begin{eqnarray}
    \norm{\left(I-P_{\psi,\ell}\right)\Sigma^{1/2}}_\HSH^2&{}={}&\text{Tr}\left[\Sigma^{1/2}(I-P_{\psi,\ell})(I-P_{\psi,\ell})\Sigma^{1/2}\right]\nonumber
    \\
    &{}={}&\text{Tr}\left[(I-P_{\psi,\ell})\Sigma\right]=\sum_{i\ge 1}\lambda_i-\inner{P_{\psi,\ell}}{\Sigma}_\HSH.\label{Eq:tempo}
\end{eqnarray}
Clearly the l.h.s.~of \eqref{Eq:tempo} is minimized if and only if $\inner{P_{\psi,\ell}}{\Sigma}_\HSH$ is maximized over $Q$, which occurs only when $\psi_i=\phi_i$, yielding $\inner{P_{\psi,\ell}}{\Sigma}_\HSH=\sum^\ell_{i=1}\lambda_i$.
\\\\
\noindent $(iii)$ \emph{Upper Bound:} We first establish notation for the uncentered covariance operator, 
$$C:=\int_\X k(\cdot,x)\oh k(\cdot,x)d\Pb(x),$$
and its empirical estimate 
$$C_n=\frac{1}{n}\sum_{i=1}^nk(\cdot,X_i)\oh k(\cdot,X_i),$$
which will be necessary for our upcoming analysis. We decompose the reconstruction error as
\begin{eqnarray}
    R_{nys}^\ell(\sigh)&{}={}&\E\norm{(k(\cdot,X)-m_\Pb)-P^\ell_{nys}(\sigh)(k(\cdot,X)-\widehat{m}_\Pb)}_\Hk^2\nonumber
    \\
    &{}={}&\underbrace{\E\norm{(I-P_{nys}^\ell(\sigh))\kbar}_\Hk^2}_{\circled{C}}+\underbrace{\norm{P_{nys}^\ell(\sigh)(m_\Pb-\widehat{m}_\Pb)}_\Hk^2}_{\circled{D}}\nonumber
    \\
    &{}{}&\qquad-2\E\inner{(I-P_{nys}^\ell(\sigh))\kbar}{P_{nys}^\ell(\sigh)(m_\Pb-\widehat{m}_\Pb)}_\Hk.\label{eq:nys main 3.1}
\end{eqnarray}
\noindent Now $\E[\kbar]=0$ implies the third term in (\ref{eq:nys main 3.1}) is 0. \circled{D} can be bound by writing
$$\circled{D}\le\norm{P_{nys}^\ell(\sigh)}_\OPH^2\norm{m_\Pb-\widehat{m}_\Pb}_\Hk^2\le\norm{m_\Pb-\widehat{m}_\Pb}_\Hk^2,$$
\noindent and applying Lemma \ref{lem:kernel mean bnd}, yields
\begin{equation}\label{eq:nys main 3 D bnd}
    \Pb^n\left\{\circled{D}\le\frac{32\kappa\log\frac{2}{\delta}}{n}\right\}\ge1-\delta.
\end{equation}
\noindent Regarding \circled{C}, for any $t>0$, we have
\begin{eqnarray}
    \circled{C}&{}\stackrel{(\star)}{=}{}&\norm{(I-P_{nys}^\ell(\sigh))\Sigma^{1/2}}_\HSH^2=\norm{(I-P_{nys}^\ell(\sigh))(\sigh+tI)^{1/2}(\sigh+tI)^{-1/2}\Sigma^{1/2}}_\HSH^2\nonumber
    \\
    &{}\le{}&\norm{(I-P_{nys}^\ell(\sigh))(\sigh+tI)^{1/2}}_\OPH^2\norm{(\sigh+tI)^{-1/2}\Sigma^{1/2}}_\HSH^2\nonumber
    \\
    &{}\le{}&\norm{(I-P_{nys}^\ell(\sigh))(\sigh+tI)^{1/2}}_\OPH^2\norm{(\sigh+tI)^{-1/2}(\Sigma+tI)^{1/2}}_\OPH^2\norm{(\Sigma+tI)^{-1/2}\Sigma^{1/2}}_\HSH^2\nonumber
    \\
    &{}\stackrel{(\dag)}{\le}&2\Cal{N}_\Sigma(t)\norm{(I-P_{nys}^\ell(\sigh))(\sigh+tI)^{1/2}}_\OPH^2\label{eq:nys main thm 3.2},
\end{eqnarray}
where we have used Lemma \ref{lem:recons rewrite 1} in $(\star)$ and Lemma \ref{lem:op bounds} in $(\dag)$. For convenience, we now let $\sigh_t=\sigh+tI$. Observing the last term, we have
\begin{eqnarray}
    \norm{(I-P_{nys}^\ell(\sigh))(\sigh+tI)^{1/2}}_\OPH^2&{}\le{}&2\norm{\left(I-\bar{P}_m\right)\sigh_t^{1/2}}_\OPH^2+2\norm{\left(\bar{P}_m-P_{nys}^\ell(\sigh)\right)\sigh_t^{1/2}}_\OPH^2\nonumber
    \\
    &{}\le{}&2\underbrace{\norm{(I-\bar{P}_m)\Sigma_t^{1/2}}_\OPH^2\norm{\Sigma_t^{-1/2}\sigh_t^{1/2}}_\OPH^2}_{\circled{C1}}\nonumber
    \\
    &{}{}&\hspace{7mm}+2\norm{(I-P_{nys}^\ell(\sigh))\bar{P}_m\sigh_t^{1/2}}_\OPH^2\label{eq:nys main thm 3.3}\\
    &{}={}&2\,\circled{C1}+2\norm{(I-P_{nys}^\ell(\sigh))\bar{P}_m\sigh_t\bar{P}_m(I-P_{nys}^\ell(\sigh))}_\OPH\nonumber
    \\
    &{}\le{}&2\,\circled{C1}+2\norm{(I-P_{nys}^\ell(\sigh))\bar{P}_m\sigh\bar{P}_m(I-P_{nys}^\ell(\sigh))}_\OPH\nonumber
    \\
    &{}{}&\qquad+2t\norm{(I-P_{nys}^\ell(\sigh))\bar{P}_m(I-P_{nys}^\ell(\sigh))}_\OPH\nonumber
    \\
    &{}\le{}&2\,\circled{C1}+2\left(\lambdatild_{\ell+1}+t\right),\label{eq:nys main thm 3.4}
\end{eqnarray}
\noindent where we have used $\text{ran}(P_{nys}^\ell(\sigh))\subset\text{ran}(\bar{P}_m)$ in (\ref{eq:nys main thm 3.3}), and (\ref{eq:nys main thm 3.4}) holds because $P_{nys}^\ell(\sigh)$ projects onto the $\ell$-eigenspace of $\bar{P}_m\sigh\bar{P}_m$. Lemmas \ref{rudi lem 6 u-stat} and \ref{lem:op bounds}\emph{(iii)} give
\begin{equation}
    \Pb^n\left\{\circled{C1}\le3t\right\}\ge1-4\delta\label{eq:nys main thm 3 C1 bnd}.
\end{equation}
\noindent Continuing, we have
\begin{eqnarray}
    \lambdatild_{\ell+1}+t\le|\lambdatild_{\ell+1}-\lambdah_{\ell+1}|+\lambdah_{\ell+1}+t\stackrel{(\dag)}{\le}\frac{1}{n(n-1)}\norm{\left(\tilde{\mathbf{K}}-\mathbf{K}\right)\mathbf{H}_n}_\OPRn+\lambdah_{\ell+1}+t,\label{eq:nys main thm 3.5}
\end{eqnarray}
where $(\dag)$ uses the Hoffman-Wielandt inequality \citep{Bhatia-94} because $\lambdatild_{\ell+1}$ (\emph{resp.} $\lambdah_{\ell+1}$) is an eigenvalue of $\tilde{\mathbf{K}}\mathbf{H}_n$ (\emph{resp.} $\mathbf{KH}_n$). Letting $P_m$ to be the orthogonal projector onto $\text{span}\{k(\cdot,X_{r_j})|j\in[m]\}$, it is easy to verify that $P_m=\Tilde{S}_m^*\mathbf{K}_{mm}^{-1}\Tilde{S}_m$ \citep[Lemma 1]{Rudi-15}.  Using $S_n\Tilde{S}_m^*=\mathbf{K}_{nm}$, which follows from Lemma \ref{lem:sampling op}\emph{(ii)}, we have the expression
\begin{equation*}
    \tilde{\mathbf{K}}=\mathbf{K}_{nm}\mathbf{K}_{mm}^{-1}\mathbf{K}_{mn}=S_n\Tilde{S}_m^*\mathbf{K}_{mm}^{-1}\Tilde{S}_mS_n^*=S_nP_mS_n^*.\nonumber
\end{equation*}
Thus, we may write,
\begin{eqnarray}
    \norm{\left(\tilde{\mathbf{K}}-\mathbf{K}\right)\mathbf{H}_n}_\OPRn&{}\le{}&\norm{\tilde{\mathbf{K}}-\mathbf{K}}_\OPRn\norm{\mathbf{H}_n}_\OPRn\nonumber
    \\
    &{}={}&n\norm{S_n(I-P_m)S_n^*}_\OPRn\nonumber
    \\
    &{}={}&n\norm{(I-P_m)S_n^*S_n(I-P_m)}_\OPH\nonumber
    \\
    &{}={}&n^2\norm{(I-P_m)C_n(I-P_m)}_\OPH,\label{eq:nys main thm 3.6}
\end{eqnarray}
\noindent where we have used Lemma \ref{lem:sampling op}\emph{(iv)} in (\ref{eq:nys main thm 3.6}).  Proceeding,
\begin{eqnarray}
(\ref{eq:nys main thm 3.6})&{}={}&n^2\norm{C_n^{1/2}(I-P_m)^2C_n^{1/2}}_\OPH\nonumber
\\
&{}\le&{}n^2\norm{C_n^{1/2}(C+tI)^{-1/2}}_\OPH^2\norm{(C+tI)^{1/2}(I-P_m)}_\OPH^2\nonumber
\\
&{}\le{}&n^2\norm{C_n^{1/2}(C_n+tI)^{-1/2}}_\OPH^2\norm{(C_n+tI)^{1/2}(C+tI)^{-1/2}}_\OPH^2\nonumber\\
&{}{}&\qquad\times\norm{(C+tI)^{1/2}(I-P_m)}_\OPH^2\nonumber
\\
&{}\le&n^2\norm{(C_n+tI)^{1/2}(C+tI)^{-1/2}}_\OPH^2\norm{(C+tI)^{1/2}(I-P_m)}_\OPH^2.\label{eq:nys main thm 3.7}
\end{eqnarray}
Applying Lemmas \ref{rudi lem 5} and \ref{rudi lem 6} to (\ref{eq:nys main thm 3.7}) and Lemma \ref{lem:op bounds}\emph{(iv)} in (\ref{eq:nys main thm 3.5}) gives
\begin{equation}
    \Pb^n\left\{\lambdatild_{\ell+1}+t\le\frac{9n^2t}{2n(n-1)}+\frac{3}{2}(\lambda_{\ell+1}+t)\right\}\ge1-4\delta.\label{eq:nys main thm 3 C2 bnd}
\end{equation}
Combining (\ref{eq:nys main thm 3 C2 bnd}) with (\ref{eq:nys main thm 3 C1 bnd}) in (\ref{eq:nys main thm 3.4}) gives 
\begin{equation}\label{eq:nys main thm 3.8}
    \Pb^n\left\{\norm{(I-P_{nys}^\ell(\sigh))(\sigh+tI)^{1/2}}_\OPH^2\le 27t+3\lambda_{\ell+1}\right\}\ge1-8\delta,
\end{equation}
where we note that $\frac{1}{n-1}\le\frac{2}{n}$ for $n\ge2$. The result follows by combining (\ref{eq:nys main thm 3.8}), (\ref{eq:nys main thm 3.2}), and (\ref{eq:nys main 3 D bnd}) in (\ref{eq:nys main 3.1}).
\vspace{.5mm}\\
\noindent \emph{Lower Bound:} Using (\ref{eq:nys main 3.1}) we have
\begin{eqnarray*}
     R_{nys}^\ell(\sigh)&{}={}&\circled{C}+\circled{D}\ge\norm{\left(I-P_{nys}^\ell(\sigh)\right)\Sigma^{1/2}}_\HSH^2.\label{eq:nys main 3 lwr}\nonumber
\end{eqnarray*}
\noindent Since we have shown in \emph{(ii)} that
$$R^\ell(\Sigma)=\inf_{
\{\psi_i\}_i\in Q}\norm{\left(I-P_{\psi,\ell}\right)\Sigma^{1/2}}_\HSH^2,$$
the lower bound follows immediately.

\subsection{Proof of Corollary \ref{poly decay corollary}}\label{sec:poly decay proof}
\emph{(i)} From Theorem \ref{thm:nys main theorem}\emph{(i)} we have 
$$R^\ell(\Sigma)=
\sum_{i>\ell}\lambda_i\lesssim\sum_{i>\ell} i^{-\alpha}\lesssim\int_\ell^\infty x^{-\alpha}dx\lesssim\ell^{1-\alpha}=n^{-\theta(1-\frac{1}{\alpha})}.$$
\noindent Similarly,
$$
R^\ell(\Sigma)=\sum_{i>\ell}\lambda_i\gtrsim\sum_{i>\ell} i^{-\alpha}\gtrsim\int_{\ell+1}^\infty x^{-\alpha}dx\gtrsim(\ell+1)^{1-\alpha}=n^{-\theta(1-\frac{1}{\alpha})}.$$
\noindent \emph{(ii)} Theorem \ref{thm:nys main theorem}(\emph{ii}) and \citep[Lemma A.8]{sriperumbudur-20} yield
$$R^\ell(\sigh)\lesssim_{\Pb^n}t^{-1/\alpha}(n^{-\theta}+t)+\frac{1}{n},$$
for $\frac{\log n}{n}\lesssim t\lesssim 1$ where we have used $\lambda_\ell\lesssim\ell^{-\alpha}=n^{-\theta}$.  Now if $\theta<1$ 
$$\inf\left\{t^{-1/\alpha}(n^{-\theta}+t)+\frac{1}{n}:\frac{\log n}{n}\lesssim t\lesssim 1\right\}\lesssim n^{-\theta(1-\frac{1}{\alpha})}+\frac{1}{n},$$
and for $\theta\ge1$
$$\inf\left\{t^{-1/\alpha}(n^{-\theta}+t)+\frac{1}{n}:\frac{\log n}{n}\lesssim t\lesssim 1\right\}\lesssim\left(\frac{\log n}{n}\right)^{\frac{\alpha-1}{\alpha}}+\frac{1}{n},$$
yielding the result.\vspace{1mm}

\noindent (\emph{iii}) Theorem \ref{thm:nys main theorem}(\emph{iii}) and \citep[Lemma A.8]{sriperumbudur-20} yield
$$R_{\widehat{\Sigma},\ell}\lesssim_{\Pb^n}t^{-1/\alpha}(n^{-\theta}+t)+\frac{1}{n},$$
with $\frac{\log n}{n}\lesssim t\lesssim 1$ and $m\gtrsim\left(\frac{1}{t}\vee\Cal{N}_{C,\infty}(t)\right)\log\frac{1}{t}$. Since $\Cal{N}_{C,\infty}(t)\lesssim\frac{1}{t}$, we have $m\gtrsim\frac{1}{t}\log\frac{1}{t}$, and the result follows as in (\emph{ii}). 

\subsection{Proof of Theorem \ref{thm:nys main theorem l2}}\label{sec:nys main theorem proof l2}
\emph{Upper Bound:} We have
\begin{eqnarray}
    T_{nys}^\ell(\sigh)&{}={}&\E\norm{\id\kbar-\id P^\ell_{nys}(\sigh)(k(\cdot,X)-\widehat{m}_\Pb)}_\lp^2\nonumber
    \\
    &{}={}&\underbrace{\E\norm{\id(I-P_{nys}^\ell(\sigh))\kbar}_\lp^2}_{\circled{E}}+\underbrace{\norm{\id P_{nys}^\ell(\sigh)(m_\Pb-\widehat{m}_\Pb)}_\lp^2}_{\circled{F}}\nonumber
    \\
    &{}{}&\qquad-2\E\inner{\id(I-P_{nys}^\ell(\sigh))\kbar}{\id P_{nys}^\ell(\sigh)(m_\Pb-\widehat{m}_\Pb)}_\lp.\label{eq:nys lp 1.1}
\end{eqnarray}
The third term of (\ref{eq:nys lp 1.1}) is 0, because $\E[\kbar]=0$.  Using $\Sigma=\id^*\id$ \citep[Proposition B.2 (\emph{iii})]{sriperumbudur-20} we write
\begin{eqnarray}
     \circled{E}&{}={}&\E\norm{\id(I-P_{nys}^\ell(\sigh))\kbar}_\lp^2=\inner{\Sigma(I-P_{nys}^\ell(\sigh))\kbar}{(I-P_{nys}^\ell(\sigh))\kbar}_\Hk\nonumber
     \\
     &{}\stackrel{(\dag)}{=}{}&\norm{\Sigma^{1/2}(I-P_{nys}^\ell(\sigh))\Sigma^{1/2}}_\HSH^2,\label{eq:nys lp 1.2}
\end{eqnarray}
where $(\dag)$ follows from Lemma \ref{lem:recons rewrite 1}. Now (\ref{eq:nys lp 1.2}) is similar to \circled{C} in the proof of Theorem \ref{thm:nys main theorem}(\emph{iii}), and the proof will proceed similarly.  Using a similar argument to that in (\ref{eq:nys main thm 3.2}), and the idempotency of $I-P_{nys}^\ell(\sigh)$, we have
\begin{eqnarray}
    \norm{\Sigma^{1/2}(I-P_{nys}^\ell(\sigh))\Sigma^{1/2}}_\HSH^2&{}\le{}&\Cal{N}_\Sigma(t)\norm{\Sigma_t^{-1/2}\Sigma^{1/2}}_\OPH^2\norm{\Sigma_t^{1/2}(I-P_{nys}^\ell(\sigh))}_\OPH^4\nonumber
    \\
    &{}\le{}&\Cal{N}_\Sigma(t)\norm{\sigh_t^{-1/2}\Sigma_t^{1/2}}_\OPH^4\norm{(I-P_{nys}^\ell(\sigh))\sigh_t^{1/2}}_\OPH^4\label{eq:nys lp 1.3}.
\end{eqnarray}
The last term in (\ref{eq:nys lp 1.3}) is simply the square of the last term of (\ref{eq:nys main thm 3.2}); therefore, we simply apply the result from (\ref{eq:nys main thm 3.8}), yielding
\begin{equation}\label{eq:nys lp 1.6}
    \Pb^n\left\{\norm{\sigh_t^{1/2}(I-P_{nys}^\ell(\sigh))}_\OPH^4\le 9(9t+\lambda_{\ell+1})^2\right\}\ge1-8\delta.
\end{equation}
Continuing, 
\begin{eqnarray}
    \circled{F}&{}={}&\norm{\Sigma^{1/2}P_{nys}^\ell(\sigh)(m_\Pb-\widehat{m}_\Pb)}_\Hk^2\le\norm{\Sigma}_\OPH^2\norm{P_{nys}^\ell(\sigh)}_\OPH^2\norm{m_\Pb-\widehat{m}_\Pb}_\Hk^2\nonumber
    \\
    &{}\le{}&\kappa\norm{m_\Pb-\widehat{m}_\Pb}_\Hk^2\le\frac{32\kappa^2\log\frac{2}{\delta}}{n},\label{eq:nys lp 1.5}
\end{eqnarray}
where last inequality holds with probability at least $1-\delta$ from Lemma \ref{lem:kernel mean bnd}.  The result follows by applying Lemma \ref{lem:op bounds}(\emph{ii}) to the middle term in (\ref{eq:nys lp 1.3}) and combining with (\ref{eq:nys lp 1.6}) and (\ref{eq:nys lp 1.5}).\vspace{.5mm}
\\
\emph{Lower Bound:} The proof of \citet[Theorem 2(\emph{ii})]{sriperumbudur-20} gives 
$$\sum_{i>\ell}\lambda_i^2=\inf_{
\{\psi_i\}_i\in Q}\norm{\Sigma^{1/2}\left(I-P_{\psi,\ell}\right)\Sigma^{1/2}}_\HSH^2.$$
The result therefore follows by noticing that $T_{nys}^\ell(\sigh)\ge\circled{E}=\norm{\Sigma^{1/2}\left(I-P_{nys}^\ell(\sigh)\right)\Sigma^{1/2}}_\HSH^2$.

\subsection{Proof of Corollary \ref{poly decay corollary lp}}\label{sec:poly decay corollary lp}
\emph{(i), (ii), (iv)} are provided in \citep[Corollary 3]{sriperumbudur-20}.  
\\
\noindent \emph{(iii)} The lower bound follows immediately from previous results. For the upper bound, Theorem \ref{thm:nys main theorem l2} and the proof of Corollary \ref{poly decay corollary} \emph{(ii)} yield
\begin{equation}\label{eq:poly cor lp 1}
    T_{nys}^\ell(\sigh)\lesssim_{\Pb^n}t^{-1/\alpha}(t+n^{-\theta})^2+\frac{1}{n},
\end{equation}
for $\frac{\log n}{n}\lesssim t\lesssim 1$ and $m\gtrsim\left(\frac{1}{t}\vee\Cal{N}_{C,\infty}(t)\right)\log\frac{1}{t}$. Now $\Cal{N}_{C,\infty}(t)=\sup_{x\in\X}\inner{k(\cdot,x)}{(C+tI)^{-1}k(\cdot,x)}_\Hk\lesssim\frac{1}{t}$; thus, $m\gtrsim\frac{1}{t}\log\frac{1}{t}$. Larger values of $t$ correspond to smaller requirement on $m$; thus, to optimize the performance of NY-EKPCA we select the largest value of $t$ such that the behavior of (\ref{eq:poly cor lp 1}) matches that of $T^\ell(\sigh)$. Setting $t=n^{-\theta}$ when $\theta<\frac{\alpha}{2\alpha-1}$ and $t=n^{-\frac{\alpha}{2\alpha-1}}$ when $\theta\ge\frac{\alpha}{2\alpha-1}$ yields the result.

\section*{Acknowledgments}
BKS is supported by National Science Foundation (NSF) CAREER Award DMS-1945396.
\bibliographystyle{apalike}
\bibliography{Comps_Bib}

\appendix

\section{Technical Results}\label{sec:technical results}
The following are a collection of technical results that are needed to prove the main results of this paper.

\begin{lemma}[\citealp{sriperumbudur-20}, Lemma A.1]\label{lem:op bounds}
Let $H$ be a separable Hilbert space and $X$ a separable topological space.  Define 
$$\mathfrak{C}=\frac{1}{2}\int_{X\times X}\left(\nu(x)-\nu(y)\right)\otimes_H\left(\nu(x)-\nu(y)\right)dP(x,y)$$
where $\nu:X\rightarrow H$ is a Bochner measurable function with $\sup_{x\in X}\norm{\nu(x)}=\kappa$.  For $\{X_i\}_{i=1}^r\stackrel{\iid}{\sim}P$ with $r\ge2$, define
$$\widehat{\mathfrak{C}}=\frac{1}{2r(r-1)}\sum_{i\neq j}^r\left(\nu(X_i)-\nu(X_j)\right)\otimes_H\left(\nu(X_i)-\nu(X_j)\right).$$
Then the following hold for any $0\le\delta\le\frac{1}{2}$ and $\frac{140\kappa}{r}\log\frac{16\kappa r}{\delta}\le t\le\norm{\mathfrak{C}}_{\Cal{L}^\infty(H)}$:
\begin{itemize}
    \item[(i)] $P^r\left\{\norm{(\mathfrak{C}+tI)^{-1/2}(\widehat{\mathfrak{C}}-\mathfrak{C})(\frak{C}+tI)^{-1/2}}_{\Cal{L}^\infty(H)}\le\frac{1}{2}\right\}\ge1-2\delta;$
    \item[(ii)] $P^r\left\{\sqrt{\frac{2}{3}}\le\norm{(\frak{C}+tI)^{1/2}(\widehat{\frak{C}}+tI)^{-1/2}}_{\Cal{L}^\infty(H)}\le\sqrt{2}\right\}\ge1-2\delta;$
    \item[(iii)] $P^r\left\{\norm{(\frak{C}+tI)^{-1/2}(\widehat{\frak{C}}+tI)^{1/2}}_{\Cal{L}^\infty(H)}\le\sqrt{\frac{3}{2}}\right\}\ge1-2\delta;$
    \item[(iv)] $P^r\left\{\lambda_k(\widehat{\frak{C}})+t\le\frac{3}{2}(\lambda_k(\frak{C})+t)\right\}\ge1-2\delta$ for all $k\ge1$;
    \item[(v)]$P^r\left\{\lambda_k(\frak{C})+t\le2(\lambda_k(\widehat{\frak{C}})+t)\right\}\ge1-2\delta$ for all $k\ge1$.
\end{itemize}
\end{lemma}

\begin{lemma}[\citealp{Sterge-20}, Lemma A.1]\label{rudi lem 5}
Suppose Assumption \ref{first kern assum} holds and $\frac{9\kappa}{n}\log\frac{n}{\delta}\le t\le\norm{C}_\OPH$ for any $0<\delta<1$. Then
$$\Pb^n\left\{\norm{(C_n+tI)^{1/2}(C+tI)^{-1/2}}_\OPH\le\sqrt{\frac{3}{2}}\right\}\ge1-\delta,$$
where $C=\int k(\cdot,x)\otimes_\Cal{H} k(\cdot,x)\,d\bb{P}(x)$.
\end{lemma}

\begin{lemma}[\citealp{sriperumbudur-20}, Lemma A.4\emph{(i)}]\label{lem:kernel mean bnd}
Suppose Assumption \ref{first kern assum} holds and $n\ge2\log\frac{2}{\delta}$ for any $0<\delta<1$. Then
$$\Pb^n\left\{\norm{m_\Pb-\widehat{m}_\Pb}_\Hk^2\le\frac{32\kappa\log\frac{2}{\delta}}{n}\right\}\ge1-\delta.$$
\end{lemma}

\begin{lemma}[\citealt{Rudi-15}, Lemma 6]\label{rudi lem 6}
Suppose Assumption \ref{first kern assum} holds, and for some $m<n$, the set of indices $\{i_1,\cdots,i_m\}$ is drawn uniformly without replacement from $[n]$. For some $t>0$, define $\Cal{N}_{C,\infty}(t)=\sup_{x\in\X}\inner{k(\cdot,x)}{(C+tI)^{-1}k(\cdot,x)}_\Hk$, where $C=\int_\X k(\cdot,x)\oh k(\cdot,x)\,d\Pb(x)$ is the uncentered covariance operator. Then, for any $\delta>0$ and $m\ge(67\vee5\Cal{N}_{C,\infty}(t))\log\frac{4\kappa}{t\delta}$, we have
$$\Pb^n\left\{\norm{(I-P_m)(C+tI)^{1/2}}_\OPH^2\le3t\right\}\ge1-\delta,$$
where $P_m$ is the orthogonal projector onto $\emph{span}\{k(\cdot,X_{i_j})|j\in[m]\}$. \end{lemma}

\noindent The following is an adaption of \citet[Lemma 6]{Rudi-15} for U-statistics.
\begin{lemma}\label{rudi lem 6 u-stat}
Suppose Assumption \ref{first kern assum} holds, and for some $m<n$, the set of indices $\{i_j\}_{j=1}^{m}$ is drawn uniformly from the set of all partitions of size $m$ of $\{1,2,...,n\}$. Then, for $0\le\delta\le\frac{1}{2}$,  $0<t\le\norm{\Sigma}_{\Cal{L}^\infty(H)}$ and $m\ge \frac{140\kappa}{t}\log\frac{8}{t\delta}$, we have
$$\Pb^n\left\{\norm{(I-\bar{P}_m)(\Sigma+tI)^{1/2}}_\OPH^2\le2t\right\}\ge1-2\delta,$$
where $\bar{P}_m$ is the orthogonal projector onto $\bar{\Hk}_m$ as defined in Section \ref{sec:NY-EKPCA}.
\end{lemma}
\begin{proof}
Define $\sigh_m:=\frac{1}{2m(m-1)}\sum_{j\ne l}^m(k(\cdot,X_{i_j})-k(\cdot,X_{i_l}))\oh(k(\cdot,X_{i_j})-k(\cdot,X_{i_l}))$, which means
 $\sigh_m=\frac{1}{2m(m-1)}\tilde{S}_m^*\mathbf{H}_m\tilde{S}_m=\frac{1}{2(m-1)}\tilde{S}_m^*\mathbf{C}_m^2\tilde{S}_m=Z^*Z$, where $Z^*=\frac{1}{\sqrt{2(m-1)}}\tilde{S}_m^*\mathbf{C}_m$. Note that $Z^*$ has range $\bar{\Hk}_m$, and so $\text{ran}(\bar{P}_m)=\text{ran}(Z^*)$. Therefore, by Proposition 3 of \cite{Rudi-15}, we have
$$\norm{(I-\bar{P}_m)(\Sigma+tI)^{1/2}}_\OPH^2\le t\norm{(\sigh_m+tI)^{-1/2}(\Sigma+t)^{1/2}}_\OPH^2.$$
The proof is completed by applying Lemma \ref{lem:op bounds}.
\end{proof}

\begin{lemma}\label{lem:recons rewrite 1}
For any orthogonal projector $P:\Hk\rightarrow\Hk$, the following holds:
$$\E_{X\sim\Pb}\norm{(I-P)(k(\cdot,X)-m_\Pb)}_\Hk^2=\norm{(I-P)\Sigma^{1/2}}_\HSH^2.$$
\end{lemma}
\begin{proof}
By denoting $\kbar:=k(\cdot,X)-m_\Pb$, we have
\begin{eqnarray*}
    \E\norm{(I-P)\kbar}_\Hk^2&{}={}&\E\inner{(I-P)\kbar}{(I-P)\kbar}_\Hk\nonumber
    \\
    &{}={}&\E\inner{(I-P)\kbar}{\kbar}_\Hk\nonumber
    \\
    &{}={}&\E\inner{I-P}{\kbar\oh\kbar}_\HSH,\nonumber
\end{eqnarray*}
where we used $(I-P)^2=(I-P)$. Since $k$ is bounded, and thus Bochner integrable, it follows that 
\begin{eqnarray}
    \E\inner{I-P}{\kbar\oh\kbar}_\HSH&{}={}&\inner{I-P}{\E[\kbar\oh\kbar]}_\HSH\nonumber
    \\
    &{}={}&\inner{I-P}{\Sigma}_\HSH=\text{tr}\left((I-P)\Sigma\right)\nonumber
    \\
    &{}={}&\text{tr}\left(\Sigma^{1/2}(I-P)^2\Sigma^{1/2}\right).\label{eq:recons rewrite 1}
\end{eqnarray}
The proof is completed by using $\norm{(I-P)\Sigma^{1/2}}_\HSH^2=\text{tr}\left(\Sigma^{1/2}(I-P)^2\Sigma^{1/2}\right)$ in (\ref{eq:recons rewrite 1}).
\end{proof}

The following is a collection of results regarding the sampling and approximate sampling operators, $S_n$ and $\Tilde{S}_m$ respectively, quoted from \citet{sriperumbudur-20} and \citet{Rudi-15}.

\begin{lemma}\label{lem:sampling op}
The sampling operator, $S_n:\Hk\rightarrow\R^n,\,f\mapsto\left(f(X_1),f(X_2),\ldots,f(X_n)\right)^\top$, and approximate sampling operator, $\Tilde{S}_m:\Hk\rightarrow\R^m,\,f\mapsto\left(f(X_{i_1}),f(X_{i_2}),\ldots,f(X_{i_m})\right)^\top$, have the following properties:
\begin{itemize}
    \item[(i)] $S_n^*:\R^n\rightarrow\Hk,\,\alpha\mapsto\sum_{i=1}^n\alpha_ik(\cdot,X_i)$; 
    \item[(ii)] $\frac{1}{n(n-1)}S_n^*\mathbf{H}_nS_n=\sigh$;
    \item[(iii)] $S_nS_n^*=\mathbf{K}$ and $\Tilde{S}_m\Tilde{S}_m^*=\mathbf{K}_{mm}$;
    \item[(iv)] $\frac{1}{n}S_n^*S_n=\frac{1}{n}\sum_{i=1}^n k(\cdot,X_i)\oh k(\cdot,X_i)=C_n$;
    \item[(v)] $\mathbf{K}_{nm}=S_n\Tilde{S}_m^*$.
\end{itemize}
\end{lemma}
\begin{proof}
We refer the reader to \citet[Proposition B.1]{sriperumbudur-20} for the proofs of \emph{(i), (ii),}  and \emph{(iii)}. The proofs of \emph{(iv)} and \emph{(v)} are provided in \citet[Section B]{Rudi-15}.
\end{proof}
\end{document}